\documentclass[11pt]{article}

\usepackage{mathrsfs,amsmath,amsfonts,amssymb,amstext,bm,bbm,dsfont,pifont,amscd,
            amsthm,stmaryrd,euscript,color,xcolor}

\usepackage{algorithmic}
\usepackage{algorithm}
\usepackage{url}
\usepackage{mathtools}
\usepackage{epsfig}
\usepackage{float}
\usepackage{appendix}
\usepackage{amstext}
\usepackage{floatflt}
\usepackage{nicefrac}
\usepackage{amsmath}
\usepackage{palatino}
\usepackage{anysize,hyperref}
\usepackage{enumerate}
\usepackage{epsfig}
\usepackage{hyperref}
\usepackage{bm}
\usepackage{bbm}
\usepackage{graphicx}
\numberwithin{equation}{section} 
\usepackage{mathtools}

\usepackage{graphicx}
\usepackage{amsmath, amssymb, amsthm}

\definecolor{mylila}{rgb}{0.6, 0.4, 1.0} 
\definecolor{red}{rgb}{0.75,0.0,0.0}
\definecolor{green}{rgb}{0.0,0.5,0.0}
\definecolor{bluebell}{rgb}{0.64, 0.64, 0.82}
\definecolor{amethyst}{rgb}{0.6, 0.4, 0.8}
\definecolor{applegreen}{rgb}{0.55, 0.71, 0.0}
\definecolor{awesome}{rgb}{1.0, 0.13, 0.32}
\definecolor{caribbeangreen}{rgb}{0.0, 0.8, 0.6}

\newtheorem{theorem}{Theorem}[section]
\newtheorem{definition}[theorem]{Definition}
\newtheorem{lemma}[theorem]{Lemma}
\newtheorem{proposition}[theorem]{Proposition}

\newtheorem{example}[theorem]{Example}
\newtheorem{remark}[theorem]{Remark}

\newcommand{\mtc}{\mathcal}

\def\argmin{\mathop{\rm arg\,min}\limits}

\def\cG{{\mtc{G}}}
\def\cH{{\mtc{H}}}

\def\cO{{\mtc{O}}}

\newcounter{nbnotes}
\makeatletter
\newcommand{\checknbnotes}{
\ifnum \thenbnotes > 0
\@latex@warning@no@line{**********************************************************************}
\@latex@warning@no@line{* The document contains \thenbnotes \space color note(s)}
\@latex@warning@no@line{**********************************************************************}
\fi}
\makeatother

\title{From inexact optimization to learning via gradient concentration}

\author{
Bernhard Stankewitz  \\
Humboldt University of Berlin \\
\texttt{stankebe@math.hu-berlin.de} 
\and 
Nicole M\"ucke\\
Technical University Braunschweig \\
\texttt{nicole.muecke@tu-braunschweig.de} 
\and
Lorenzo Rosasco \\
MaLGa, DIBRIS, Universit\'a di Genova\\
CBMM, MIT\\
Istituto Italiano di Tecnologia\\
\texttt{lorenzo.rosasco@unige.it} 
}

\begin{document}

\maketitle

\begin{abstract}
  Optimization in machine learning typically deals with the minimization of
  empirical objectives defined by training data.
  However, the ultimate goal of learning is to minimize the error on future data
  (test error),  for which the training data provides only  partial information.
  In this view, the optimization problems that are practically feasible are
  based on inexact quantities that are stochastic in nature.
  In this paper, we show how probabilistic results, specifically gradient
  concentration,  can be combined with results from inexact optimization to
  derive sharp test error guarantees.
  By considering unconstrained objectives we highlight the implicit
  regularization properties of optimization for learning.  
%
%
\end{abstract}

\section{Introduction}
\label{sec_Introduction}

Optimization plays a key role in modern machine learning, and is typically used
to define estimators by minimizing empirical objective functions
\cite{SraEtal2011MLOpt}.
These  objectives are  based on a data fit term, suitably penalized, or
constrained, to induce an inductive bias in the learning process
\cite{ShalevBen2014ML}. The idea is that the empirical objectives should provide an approximation to the error on future data (the test error) which is the quantity that one wishes to minimize in learning. 
The quality, of such an approximation error is typically deferred to a statistical analysis. In this view, optimization and statistical aspects are tackled separately. 

Recently, a new perspective has emerged  in machine learning showing that optimization itself can in fact be directly used to search for  a solution with small test error. 
Interestingly, no  explicit penalties/constraints are needed, since a bias in the search for a solution is implicitly enforced during the optimization process. 
This phenomenon has been called implicit regularization and it has been shown to possibly play a role in explaining the learning curves  observed in deep learning, see
for instance in \cite{GunasekarEtal2018Characterizing, Neyshabur2017Implicit}
and references therein.  
Further, implicit regularization has been advocated as a way to improve
efficiency of learning methods by tackling statistical and optimization
aspects at once
\cite{RosascoVilla2015Iterative, YangEtal2019EarlyStopping,
BlanchardEtal2018SVD, CelisseWahl2021Discrepancy}. 
As it turns out,  implicit regularization is closely related to the notion of iterative regularization with a long history in inverse problems \cite{Land51}. 

The basic example of implicit regularization is gradient descent for linear
least squares, which is well known to converge to the minimum norm least squares
solution \cite{EnglEtal1996Regularization,YaoEtal2007Early}.
The learning properties of gradient descent for least squares are now quite well
understood \cite{YaoEtal2007Early, RaskuttiEtal2014Early} including the
extension to non-linear kernelized models 
\cite{BauerEtal2007Regularization, BlanchardMuecke2018Optimal}, stochastic
gradients
\cite{Dieuleveut2017Harder, DieuleveutEtal2016Large, MueckeEtal2019Beating},
accelerated methods
\cite{BlanchardKraemer2016Convergence, PaglianaRosasco2019Implicit} and distributed
approach-es 
\cite{ZhangEtal2015Divide, MueckeBlanchard2018Parallelizing, RichardsRebeschini2020Graph}.
Much less is known when other norms or loss functions are considered.
Implicit regularization biased to more general  norms have been considered for
example in \cite{Vavskevivcius2020Statistical,Vil17}.
Implicit regularization for loss functions other than the square loss have been
considered in a limited number of works. 
There is a vast literature on stochastic gradients techniques, see e.g. \cite{MueckeEtal2019Beating} and references therein, but these analyses do not apply  when (batch) training error gradients are used, which is the focus in this work.
The logistic loss functions for classification has recently been considered both
for linear and non-linear models, see for example
\cite{Soudry2018Implicit, JiTelgarsky2019Implicit}.
Implicit regularization for general convex Lipchitz loss with linear and kernel
models have been first considered in \cite{LinEtal2016Iterative} 
for subgradient methods and in
\cite{LinEtal2016Generalization} for stochastic gradient methods but only with suboptimal rates. 
Improved rates have been provided in \cite{YangEtal2019EarlyStopping} for
strongly convex losses and more recently in \cite{LeiEtal2021Generalization}
with a general but complex analysis.  A stability based approach, in the sense of \cite{BouEli02}, is  studied in \cite{ChenEtal2018Stability}.
  
In this paper, we further push this line of work considering implicit
regularization for linear models with convex, Lipschitz and smooth loss
functions based on gradient descent. Indeed, for this setting we derive sharp rates considering both the last and the average iterate. 
Our approach highlights a proof technique which is less common in learning and is directly based on a combination of optimization and statistical results.
The usual approach in learning theory is to derive optimization results for empirical objectives and then use statistical arguments to asses to which extent the empirical objectives approximate the test error that one ideally wished to minimize, see e.g.  \cite{ShalevBen2014ML}. Instead, we  view the empirical gradient iteration 
as the inexact version of the gradient iteration for the test error. This allows
to apply results from inexact optimization, see
e.g.\cite{BertsekasTsitsiklis2000Gradient, SchmidtEtal2011Inexact}, and requires
using statistical/probabilistic arguments to asses the quality of the gradient
approximations (rather than that of the objectives functions). For this latter
purpose, we utilize recent concentration of measures results for vector valued
variables, to establish gradient concentration  \cite{Foster2018Uniform}. While
the idea of combining inexact optimization and concentration results has been
considered before \cite{GorbunovEtal2020Stochastic}, here we illustrate it in a
prominent way to highlight its usefulness. Indeed, we show that this approach
leads to sharp results for a specific but important setting and we provide some simple numerical results that illustrate and corroborate our findings.
By highlighting the key ideas in the proof techniques we hope to encourage
further results combining statistics and optimization, for example considering
other forms of gradient approximation or optimization other than the basic gradient descent.

%

The remainder of the paper is structured as follows:
In Section \ref{sec_LearningWithGradientMethodsAndImplicitRegularization}, we
collect some structural assumptions for our setting.
In Section \ref{sec_MainResultsAndDiscussion}, we formulate the  assumptions we put
on the loss function and state  and discuss the main results of the paper and 
 the novel aspects of our approach.
Section \ref{sec_FromInexactOptimisationToLearning} presents the more technical
aspects of the analysis.
In particular, we explain in detail how results from inexact optimization and
concentration of measure can be combined to come up with a new proof technique
for learning rates.
Finally, Section \ref{sec_Numerics} illustrates the key features of our
theoretical results with numerical experiments.


\section{Learning with gradient methods and implicit regularization}
\label{sec_LearningWithGradientMethodsAndImplicitRegularization}

Let $ ( \mathcal{H}, \| \cdot \| ) $ be a real, separable Hilbert space and
$ \mathcal{Y} $ a subset of $ \mathbb{R} $. 
We consider random variables $ (X, Y) $ on a probability space $ ( \Omega,
\mathscr{F}, \mathbb{P} ) $ with values in $ \mathcal{H} \times \mathcal{Y} $
and unknown distribution $ \mathbb{P}_{(X, Y)} $.
The marginal distribution of $ X $ is denoted by $ \mathbb{P}_{X} $. 
Additionally, we make the standard assumption that $ X $ is bounded.
\begin{enumerate}
  \item [{\color{red} \textbf{(A1)}}] 
    \label{ass_Bounded}
    \textbf{{\color{red} (Bound)}:} 
    We assume $ \| X \| \le \kappa $ almost surely for some 
    $ \kappa \in [ 1, \infty ) $. 
\end{enumerate}
Based on the observation of $ n $ i.i.d. copies
$ 
  ( X_{1}, Y_{1} ), \dots ( X_{n}, Y_{n} )
$ 
of $ ( X, Y ) $, we want to learn a linear relationship between $ X $
and $ Y $ expressed as an element $ w \in \mathcal{H} $.\footnote{
  Note that this includes many settings as special instances. 
  In particular, it includes the standard setting of kernel learning, see
  Appendix A in \cite{RosascoVilla2015Iterative}. 
}
For an individual observation $ ( X, Y ) $ and the choice $ w \in \mathcal{H} $,
we suffer the loss $ \ell( Y, \langle X, w \rangle ) $, where 
$ 
  \ell: \mathcal{Y} \times \mathbb{R} \to [ 0, \infty )
$ 
is a product-measurable loss function.
Our goal is to find $ w \in \mathcal{H} $ such that that the \emph{population
risk} 
$ \mathcal{L}: \mathcal{H} \to [ 0, \infty ) $ given by
\begin{align}
  \label{eq_PopulationRisk}
    \mathcal{L}(w): 
  = 
    \mathbb{E}_{(X, Y)} [ \ell(Y, \langle X, w \rangle) ] 
  = 
    \int \ell(y, \langle x, w \rangle) \, \mathbb{P}_{( X, Y )}(d ( x, y )) 
\end{align}
is small.
The observed data represent the training set, while the population risk can be interpreted as an abstraction of the concept of the test error. 

In the following, we assume that a minimizer of $ \mathcal{L} $ in 
$ \mathcal{H} $ exists.
\begin{enumerate}
  \item [{\color{red} \textbf{(A2)}}] 
    \label{ass_Min} 
    \textbf{{\color{red} (Min)}:} 
    We assume there exists some $ w_{*} \in \mathcal{H} $ such that 
    $
      w_{*} \in \argmin_{ w \in \mathcal{H} } \mathcal{L}(w)
    $.
\end{enumerate}
Note that the $ \argmin $ is taken only over $ \mathcal{H} $ and not over all
measurable functions.
Under \hyperref[ass_Min]{{\color{red} \textbf{(Min)}}}, minimizing the
population risk is equivalent to minimizing the \emph{excess risk} 
$ 
  \mathcal{L}(w) - \mathcal{L}( w_{*} ) \ge 0
$. 

In this work, we are interested in bounding the excess risk, when our choice of
$ w $ is based on  applying \emph{gradient descent} (GD) to the \emph{empirical
risk} computed from the training data,
\begin{align}
    \widehat{ \mathcal{L} }(w): 
  = 
    \frac{1}{n} 
    \sum_{j = 1}^{n} 
    \ell(Y_{j}, \langle X_{j}, w \rangle). 
\end{align}
We consider a basic gradient iteration, which is well defined when the loss
function is differentiable in the second argument with a product-measurable
derivative 
$
  \ell': \mathcal{Y} \times \mathbb{R} \to \mathbb{R} 
$.

\begin{definition}[Gradient descent algorithm]
  \label{def_GradientDescentAlgorithm}
  \
  \begin{enumerate}
    \item Choose $ v_{0} \in \mathcal{H} $ and a sequence of step sizes 
      $ ( \gamma_{t} )_{t \ge 0} $.

    \item For $ t = 0, 1, 2, \dots $, define the GD-iteration
      \begin{align}
            v_{t+1} 
        =
            v_t - \gamma_t \nabla \widehat{ \mathcal{L} }( v_t )
        = 
            v_t 
          - 
            \frac{\gamma_t}{n} 
            \sum_{j = 1}^n 
              \ell'(Y_j , \langle X_{j}, v_{t} \rangle)
              X_j.
      \end{align}
      
    \item
      For some $ T \ge 1 $, we consider both the last iterate $ v_{T} $ and
      the the averaged GD-iterate
      $
        \overline{v}_{T}: = \frac{1}{T} \sum_{t = 1}^{T} v_{t}
      $.
  \end{enumerate}
\end{definition}

\noindent 

Here, we focus on batch gradient, so that all training points are used in each
iteration. 
Unlike with stochastic gradient methods, the gradients at different iterations
are not conditionally independent. 
Indeed, the analysis of batch gradient is quite different to that of
stochastic gradient and could be a first step towards considering minibatching
\cite{LinRosasco2017Optimal, MueckeEtal2019Beating, GorbunovEtal2020Stochastic}. 
%
In our analysis, we always fix a constant step size 
$ 
  \gamma_{t} = \gamma > 0
$
for all $ t \ge 0 $ and   consider both the  average and last iterate. 
Both choices are common in the optimization literature \cite{SraEtal2011MLOpt}
and have also been studied in the context of learning with least squares
\cite{DieuleveutEtal2016Large, BlanchardMuecke2018Optimal,
MueckeEtal2019Beating}, see also our extended discussion in Subsection
\ref{ssec_DiscussionOfRelatedWork}.
In the following, we characterize the learning properties of the gradient
iteration in  Definition \ref{def_GradientDescentAlgorithm} in terms of the
corresponding excess risk. 
In particular, we derive learning bounds matching the best bounds for estimators obtained minimizing the penalized
empirical risk.  
Next, we  show that in the considered setting learning bounds can be derived by
studying suitable bias and variance terms controlled by the iteration number and
the step size.



\section{Main results and discussion}
\label{sec_MainResultsAndDiscussion}

Before stating and discussing our main results,  we introduce and comment on the
basic assumptions needed in our analysis. 
We make the following additional assumptions on the loss function.
\begin{enumerate}

  \item [{\color{red} \textbf{(A3)}}] 
    \label{ass_Convex} 
    \textbf{{\color{red} (Conv)}:} 
    We assume $ \ell: \mathcal{Y} \times \mathbb{R} \to [ 0, \infty ) $ is convex
    in the second argument.

  \item [{\color{red} \textbf{(A4)}}] 
    \label{ass_Lipschitz} 
    \textbf{{\color{red} (Lip)}:} We assume $ \ell $ to be $ L $-Lipschitz, i.e.
    for some $ L > 0 $,
    \begin{align}
      | \ell(y, a) - \ell(y, b) | \le L | a - b | 
      \qquad
      \text{ for all } y \in \mathcal{Y}, a, b \in \mathbb{R}.
    \end{align}

  \item [{\color{red} \textbf{(A5)}}] 
    \label{ass_Smooth} 
    \textbf{{\color{red} (Smooth)}:} We assume $ \ell $ to be $ M $-smooth, i.e. 
    $ \ell $ is differentiable in the second argument with product-measurable
    derivative
    $ 
      \ell': \mathcal{Y} \times \mathbb{R} \to \mathbb{R} 
    $ 
    and for some $ M > 0 $,
    \begin{align}
      \label{eq_MSmoothness}
      | \ell'(y, a) - \ell'(y, b) | \le M | a - b | 
      \qquad
      \text{ for all } y \in \mathcal{Y}, a, b \in \mathbb{R}.
    \end{align}
    Note that Equation \eqref{eq_MSmoothness} immediately implies that
    \begin{align}
      \label{eq_MSmoothnessAsQuadraticBound}
        \ell(y, b) 
      \le 
        \ell(y, a) + \ell'(y, a) ( b - a ) + \frac{M}{2} | b - a | 
      \qquad 
      \text{ for all } y \in \mathcal{Y}, a, b \in \mathbb{R},
    \end{align}
    see e.g. Lemma 3.4 in \cite{Bubeck2015Optimization}.

\end{enumerate}
For notational convenience, we state the assumptions
\hyperref[ass_Lipschitz]{{\color{red} \textbf{(Lip)}}}
and 
\hyperref[ass_Smooth]{{\color{red} \textbf{(Smooth)}}}
globally for all $ a, b \in \mathbb{R} $. 
It should be noted, however, that this is not necessary.

\begin{remark}[Local formulation of assumptions]
  \label{rem_LocalFormulationOfAssumptions}
  In our analysis, we only apply 
  \hyperref[ass_Lipschitz]{{\color{red} \textbf{(Lip)}}}
  and 
  \hyperref[ass_Smooth]{{\color{red} \textbf{(Smooth)}}}
  for arguments of the form 
  $ 
    a = \langle v, x \rangle
  $, 
  where $ \| v \| \le R $ for $ R = \max \{ 1, 3 \| w_{*} \| \} $ and $ \|
  x \| \le \kappa $  with $ \kappa $ from \hyperref[ass_Bounded]{{\color{red}
  \textbf{(Bound)}}}.
  Therefore, all of our results also apply to loss functions which satisfy
  the above assumptions for all $ a, b \in [ - \kappa R, \kappa R ] $ for
  constants $ L $ and $ M $ potentially depending on $ \kappa $ and $ R $. 
\end{remark}

\noindent In light of Remark \ref{rem_LocalFormulationOfAssumptions} our
analysis is applicable to many widely used loss functions, see e.g. Chapter 2 in
\cite{SteinwartChristmann2008SVMs}. 

\begin{example}[Loss functions satisfying the assumptions]
  \label{expl_LossFunctionsSatisfyingTheAssumptions}
  \
  \begin{enumerate}
    \item [(a)] \textbf{(Squared loss):} 
      If $ \mathcal{Y} = [ - b, b ] $ for some $ b > 0 $, then checking first
      and second derivatives yields that the loss
      $ 
        \mathcal{Y} \times [ - \kappa R, \kappa R ] 
        \ni
        ( y, a ) \mapsto ( y - a )^{2}
      $ 
      is convex, $ L $-Lipschitz with constant 
      $
        L = 2 ( b + \kappa R )
      $ 
      and $ M $-Smooth with constant $ M = 2 $. 

    \item [(b)] \textbf{(Logistic loss for regression):} 
      If $ \mathcal{Y} = \mathbb{R} $, then, analogously, the loss
      $ 
        \mathcal{Y} \times \mathbb{R} 
        \ni 
        ( y, a ) \mapsto - \log \Big( 
                                  \frac{ 4 e^{y - a} }
                                       { ( 1 + e^{y - a} )^{2} }
                                \Big) 
      $ 
      is convex, L-Lipschitz with constant $ L = 1 $ and $ M $-smooth with
      constant $ M = 1 $. 

    \item [(c)] \textbf{(Logistic loss for classification):}
      For classification problems with $ \mathcal{Y} = \{ - 1, 1 \} $,
      analogously, the loss 
      $
        \mathcal{Y} \times \mathbb{R}
        \ni
        ( y, a ) \mapsto \log( 1 + e^{- y a} ) 
      $ 
      is convex, $ L $-Lipschitz with constant $ L = 1 $ and $ M $-Smooth with 
      constant $ M = 1 / 4 $.

    \item [(d)] \textbf{(Exponential loss):} 
      For classification problems with $ \mathcal{Y} = \{ - 1, 1 \} $,
      analogously, the loss 
      $
        \mathcal{Y} \times [ - \kappa R, \kappa R ] 
        \ni
        ( y, a ) \mapsto  e^{- y a}
      $ 
      is convex, $ L $-Lipschitz with constant $ L = e^{\kappa R} $ and $ M
      $-smooth also with $ M = e^{\kappa R} $.
  \end{enumerate}
\end{example}

Under Assumption \hyperref[ass_Smooth]{{\color{red} \textbf{(Smooth)}}}, the
empirical risk $ w \mapsto \widehat{ \mathcal{L} }(w) $ is differentiable and we
have
\begin{align}
    \nabla \widehat{ \mathcal{L} }(w) 
  = 
    \frac{1}{n}\sum_{j = 1}^{n} 
    \ell'( Y_{j}, \langle X_{j}, w \rangle ) X_{j}. 
\end{align}
With Assumptions \hyperref[ass_Bounded]{{\color{red} \textbf{(Bound)}}} and
\hyperref[ass_Lipschitz]{{\color{red} \textbf{(Lip)}}}, via dominated
convergence, the same is true for the expected risk $ w \mapsto \mathcal{L}(w)
$ and we have
\begin{align}
    \nabla \mathcal{L}(w)
  = 
    \int
      \ell'(y, \langle x, w \rangle) x 
    \, \mathbb{P}_{( X, Y )}(d ( x, y )). 
\end{align}
Further, our assumptions on the loss directly translate into properties of the
risks:
\begin{enumerate}

  \item [{\color{red} \textbf{(A3')}}]
    \label{ass_RConvex} 
    \textbf{{\color{red} (R-Conv)}:} 
    Under 
    \hyperref[ass_Convex]{\textbf{{\color{red} (Conv)}}},
    both the expected and the empirical risk are convex.

  \item [{\color{red} \textbf{(A4')}}]
    \label{ass_RLipschitz}
    \textbf{{\color{red} (R-Lip)}:} 
    Under 
    \hyperref[ass_Bounded]{\textbf{{\color{red} (Bound)}}}
    and
    \hyperref[ass_Lipschitz]{\textbf{{\color{red} (Lip)}}},
    both the population and the empirical risk are Lipschitz-continuous with
    constant $ \kappa L $. 

  \item [{\color{red} \textbf{(A5')}}]
    \label{ass_RSmooth}
    \textbf{{\color{red} (R-Smooth)}:} 
    Under 
    \hyperref[ass_Bounded]{\textbf{{\color{red} (Bound})}}
    and
    \hyperref[ass_Smooth]{\textbf{{\color{red} (Smooth)}}},
    the gradient of both the population and the empirical risk is
    Lipschitz-continuous with constant $ \kappa^{2}  M $. 

\end{enumerate}
The derivation, which is straightforward, is included in Lemma
\ref{lem_PropertiesOfTheRisks} in Appendix \ref{app_ProofsForMainResults}.


\subsection{Formulation of main results}
\label{ssec_FormulationOfMainResults}


A first key result shows that, under the above assumptions, we can decompose
the excess risk for the averaged GD-iterate $ \overline{v}_{T} $ as well as
for the last iterate $ v_{T} $.  

\begin{proposition}[Decomposition of the excess risk]
  \label{prp_DecompositionOfTheExcessRisk}
  Suppose assumptions
  \hyperref[ass_Bounded]{{\color{red} \textbf{(Bound)}}},
  \hyperref[ass_Convex]{{\color{red} \textbf{(Conv)}}}
  and
  \hyperref[ass_Smooth]{{\color{red} \textbf{(Smooth)}}}
  are satisfied.
  Consider the GD-iteration from Definition \ref{def_GradientDescentAlgorithm}
  with  $T \in \mathbb{N} $ and constant step size
  $
    \gamma \le 1 / ( \kappa^{2} M )
  $ and let $ w \in \mathcal{H} $ be arbitrary. 
  \begin{enumerate}[(i)]

    \item The risk of the averaged iterate $ \overline{v}_{T} $ satisfies
    \begin{align*}
              \mathcal{L}( \overline{v}_{T} ) - \mathcal{L}(w) 
      & \le 
              \frac{1}{T} \sum_{t = 1}^{T} \mathcal{L}( v_t ) - \mathcal{L}(w) 
      \\
      & \le 
              \frac{ \| v_{0} - w \|^{2} } 
                   { 2 \gamma T } 
            + \frac{1}{T} \sum_{t = 1}^{T} 
              \langle 
                \nabla \mathcal{L}( v_{t - 1} ) 
              - 
                \nabla \widehat{ \mathcal{L} }( v_{t - 1} ), 
                v_{t} - w
              \rangle.
    \end{align*}
  
  \item [(ii)] The excess risk of the last iterate $ v_{T} $ satisfies 
    \begin{align*}
                \mathcal{L}( v_{T} ) - \mathcal{L}(w) 
      & \le 
                \frac{1}{T} \sum_{ t = 1 }^{T} 
                ( \mathcal{L}( v_{t} ) - \mathcal{L}(w) ) 
                \\
            & + 
                \sum_{ t = 1 }^{ T - 1 } 
                  \frac{1}{ t ( t + 1 ) } 
                  \sum_{ s = T - t + 1 }^{T} 
                  \langle 
                    \nabla \mathcal{L}( v_{ s - 1 } ) 
                  - \nabla \widehat{ \mathcal{L} }( v_{ s - 1 } ), 
                    v_{s} - v_{ T - t }
                  \rangle .
            \notag
    \end{align*}

  \end{enumerate}
\end{proposition}

The proof of Proposition \ref{prp_DecompositionOfTheExcessRisk} can be
found in Appendix \ref{prf_DecompositionOfTheExcessRisk}.
The above decomposition is derived using ideas from inexact optimization, in
particular results studying inexact gradients see e.g.
\cite{BertsekasTsitsiklis2000Gradient, SchmidtEtal2011Inexact}. 
Indeed, our descent procedure can be regarded as one in which the population
gradients are perturbed by the gradient noise terms 
\begin{align}
        e_{t}: 
    = 
        \nabla \widehat{ \mathcal{L} }( v_{t} ) 
      - \nabla \mathcal{L}( v_{t} ), 
    \qquad t = 1, \dots, T.
\end{align}
We further develop this discussion in Section \ref{ssec_InexactGradientDescent}. 

Note that the results above apply to any $ w \in \mathcal{H} $. 
Later we will of course set $ w = w_{*} $ from Assumption
\hyperref[ass_Min]{{\color{red} \textbf{(Min)}}}.
With this choice, Proposition \ref{prp_DecompositionOfTheExcessRisk} (i) and
(ii) provide decompositions of the excess risk into a deterministic \emph{bias
part}
\begin{align}
  \frac{ \| v_{0} - w_{*} \|^{2} }{ 2 \gamma T },
\end{align}
which can be seen as an optimization error, and a stochastic
\emph{variance part}, which is an average of the terms
\begin{align}
  \langle - e_{ t - 1 }, v_{t} - w_{*} \rangle
  \quad \text{and} \quad 
  \langle - e_{ s - 1 }, v_{s} - v_{ T - t } \rangle, 
  \qquad t = 1, \dots, T, s = T - t + 1, \dots T.
\end{align}
Note that Proposition \ref{prp_DecompositionOfTheExcessRisk} (i) can be applied
to the first sum on the right-hand side in (ii).
In order to control the bias part, it is sufficient to choose $ \gamma T $ large
enough. Controlling the variance part is more subtle and requires some care.
By Cauchy-Schwarz inequality,
\begin{align}
  \label{eq_ControllingTheVariancePart}
      \langle - e_{ t - 1 }, v_{t} - w_{*} \rangle
  \le 
      \| e_{ t - 1 } \| \| v_{t} - w_{*} \|
  \quad \text{ for all } t = 1, \dots, T.
\end{align}
A similar estimate holds for the terms $ 
  \langle - e_{ s - 1 }, v_{s} - v_{ T - t } \rangle 
$, $ s = T - t + 1, \dots T $. 
This shows that in order to upper bound the excess risk of the average gradient
iteration it is sufficient to solve two problems:
\begin{enumerate}

  \item Bound the \emph{gradient noise} terms $ 
          e_{ t - 1 }
      = 
          \nabla \widehat{ \mathcal{L} }( v_{ t - 1 } ) 
        - \nabla \mathcal{L} ( v_{ t - 1 } )
    $ in norm; 

  \item Bound the \emph{gradient path} $
      ( v_{t} )_{ t \ge 0 } 
    $ in a ball around $ w_{*} $. 

\end{enumerate}
Starting from this observation, in Proposition \ref{prp_GradientConcentration},
we state a general gradient concentration result which, for fixed $ R > 0
$, allows to derive
\begin{align}
  \label{eq_grdcoionc}
      \sup_{\| v \| \le R} 
      \| \nabla \mathcal{L}(v) - \nabla \widehat{ \mathcal{L} }(v) \| 
  \le 
      20 \kappa^{2} R ( L + M ) 
      \sqrt{ \frac{ \log(4 / \delta) }{n} } 
\end{align}
with high probability in $ \delta $ when $ n $ is sufficiently large. 
If we could prove that the gradient path $ ( v_{t} )_{ t \ge 0 } $ stays bounded,
this would  allow to control the gradient noise terms.
However, the result in Equation \eqref{eq_grdcoionc}  
itself is not enough to directly derive a bound for the gradient path.
Indeed,  in  Proposition \ref{prp_BoundedGradientPath}, we show how gradient concentration can be used to  inductively prove that with
high probability $
  \| v_{t} - w_{*} \| 
$ stays bounded by $
  R = \max \{ 1, 3 \| w_{*} \| \} 
$ for $ t \le T $ sufficiently large.
Importantly, gradient concentration thereby allows to control the
generalization error of the excess risk and the deviation of the gradient path
at the same time.
This makes this proof technique particularly appealing comparative to other
approaches in the literature, see the discussion in Sections
\ref{ssec_DiscussionOfRelatedWork} and
\ref{sec_FromInexactOptimisationToLearning}. 
Taken together, the arguments above are sufficient to prove sharp rates for
the excess risk.

\begin{theorem}[Excess Risk]
  \label{thm_ExcessRisk}
  Suppose Assumptions
  \hyperref[ass_Bounded]{{\color{red} \textbf{(Bound)}}},
  \hyperref[ass_Convex]{{\color{red} \textbf{(Conv)}}},
  \hyperref[ass_Lipschitz]{{\color{red} \textbf{(Lip)}}},
  \hyperref[ass_Smooth]{{\color{red} \textbf{(Smooth)}}}
  and
  \hyperref[ass_Min]{{\color{red} \textbf{(Min)}}}
  are satisfied.
  Let $ v_{0} = 0 $, $ T \ge 3 $ and choose a constant step size
  $
    \gamma \le \min \{ 1 / ( \kappa^{2} M ), 1 \}
  $
  in the GD-iteration from Definition \ref{def_GradientDescentAlgorithm}.
  Then for any $ \delta \in ( 0, 1 ] $, such that 
  \begin{align}
    \label{eq_ExcessRisk_nLarge}
        \sqrt{n} 
    \ge 
        \max \{ 1, 90  \gamma T \kappa^{2} ( 1 + \kappa L ) ( M + L ) \} 
        \sqrt{ \log(4 / \delta) },
  \end{align}
  the averaged iterate $ \overline{v}_{T} $ and the last iterate $ v_{T} $
  satisfy with probability at least $ 1 - \delta $, 
  \begin{align*}
            \mathcal{L}( \overline{v}_{T} ) - \mathcal{L}( w_{*} ) 
    & \le 
            \frac{\| w_{*} \|^{2}}{2 \gamma T} 
          + 
            180 \max \{ 1, \| w_{*} \|^{2} \}
            \kappa^{2} ( M + L )
            \sqrt{ \frac{ \log(4 / \delta) }{n} },
    \\
            \mathcal{L}( v_{T} ) - \mathcal{L}( w_{*} ) 
    & \le 
            \frac{\| w_{*} \|^{2}}{2 \gamma T} 
          + 
            425 \max \{ 1, \| w_{*} \|^{2} \}
            \kappa^{2} ( M + L )
            \log(T) 
            \sqrt{ \frac{ \log(4 / \delta) }{n} }.
    \notag
  \end{align*}
  In particular, setting
  $ 
    \gamma T = \sqrt{n} / ( 
                            90 \kappa^{2} ( 1 + \kappa L ) ( M + L )
                            \sqrt{\log(4 / \delta)} 
                          ) 
  $
  yields
  \begin{align*}
            \mathcal{L}(\overline{v}_{T}) - \mathcal{L}(w_{*}) 
    & \le 
            225 \max \{ 1, \| w_{*} \|^{2} \}
            \kappa^{2} ( 1 + \kappa L ) ( M + L )
            \sqrt{ \frac{ \log(4 / \delta) }{n} }, 
            \\
            \mathcal{L}( v_{T} ) - \mathcal{L}( w_{*} ) 
    & \le 
            470 \max \{ 1, \| w_{*} \|^{2} \} 
            \kappa^{2} ( 1 + \kappa L ) ( M + L )
            \log(T) 
            \sqrt{ \frac{ \log(4 / \delta) }{n} }.
  \end{align*}
\end{theorem}

\noindent The proof of Theorem \ref{thm_ExcessRisk} is in Appendix
\ref{prf_ExcessRisk}. 
Here, we comment on the above result. The gradient concentration inequality
allows to derive an explicit estimate for the variance. 
As expected, the latter improves as the number of samples increases, but
interestingly it stays bounded, provided that $\gamma T$ is not too large, see
Equation \eqref{eq_ExcessRisk_nLarge}.
Optimizing the choice of $\gamma T$ leads to the final excess risk bound. 
Such an estimate is sharp   in the sense that it matches the best available
bounds for other estimation schemes based on empirical  risk minimization with $
\ell_2$ penalties, see e.g. \cite{ShalevBen2014ML, SteinwartChristmann2008SVMs}
and references therein. We note that the average and last iterates have essentially the same performance, up to constants and logarithmic terms. 

It is worth noting that a number of different choices for the stopping time $T$
and the step size $\gamma$ are possible, as long as their product stays
constant.
Assuming that $ \kappa $ from
\hyperref[ass_Bounded]{\textbf{{\color{red} (Bound)}}} is known, the user may
choose the step size $ \gamma $ a priori when $ M $ from
\hyperref[ass_Smooth]{{\color{red} \textbf{(Smooth)}}} is known, see Example
\ref{expl_LossFunctionsSatisfyingTheAssumptions} (a), (b), (c).
When $ M $ depends on the bound $ R = \max \{ 1, 3 \| w_{*} \| \} $, see Proposition
\ref{prp_BoundedGradientPath}, the choice of $ \gamma $ must be adapted to the
norm of the minimizer $ w_{*} $, see e.g. Example
\ref{expl_LossFunctionsSatisfyingTheAssumptions} (d) and the discussion in
Remark \ref{rem_LocalFormulationOfAssumptions}.  
In this sense, it is indeed the product $ \gamma T $ that plays the role of a
regularization parameter, see also Figure \ref{fig:plots} in Section
\ref{sec_Numerics}.  

The  excess risk bound in the above theorem matches the best  bound for least
squares, obtained with an ad hoc analysis
\cite{YaoEtal2007Early, RaskuttiEtal2014Early}. 
The obtained bound improves the results  obtained in \cite{LinEtal2016Iterative}
and recovers the results in \cite{LeiEtal2021Generalization} in a special case.
Indeed, these latter results are more general and allow to derive fast rates,
however this generality is payed in terms of a considerably more complex
analysis.  In particular, our analysis allows to get explicit constants and keep the step
size constant.  More importantly, the proof we consider follows a different
path, highlighting the connection to inexact optimization. 
We further develop this point of view next. 




\subsection{Discussion of related work}
\label{ssec_DiscussionOfRelatedWork}

{\bf Comparison to the classical approach.}
  In order to better locate our work in the machine learning and statistical
  literature, we compare it with the most  important related line of research.

  In particular, we contrast our approach with the one typically used to study
  learning with gradient descent and general loss functions.
  We briefly review this latter and more classical approach. 
 The following decomposition is often considered to decompose the excess
  risk at  $ v_{t} $:
  \begin{align}
    \label{eq_ClassicExcessRiskDecomposition}
        \mathcal{L}( v_{t} ) - \mathcal{L}( w_{*} ) 
    = 
        \underbrace{
          \mathcal{L}( v_{t} ) - \widehat{ \mathcal{L} }( v_{t} )
        }_{= (\text{I})}
      + \underbrace{
          \widehat{ \mathcal{L} }( v_{t} ) - \widehat{ \mathcal{L} }( w_{*} )  
        }_{= (\text{II})}
      + \underbrace{
        \widehat{ \mathcal{L} }( w_{*} ) - \mathcal{L}( w_{*} )  
      }_{= (\text{III})},
  \end{align}
  see e.g. \cite{ShalevBen2014ML, BottouBousquet2011Tradeoffs}. 
  The second term in the decomposition can be seen as an optimization error and
  treated by deterministic results from "\emph{exact}" optimization.
  The first and last terms are stochastic and are bounded using probabilistic
  tools. 
  In particular, the first term, often called generalization error, needs some
  care.  The two more common approaches are based on stability, see e.g.
  \cite{BouEli02,ChenEtal2018Stability}, or empirical process theory
  \cite{BoucheronEtal2013Concentration}, \cite{SteinwartChristmann2008SVMs}. 
  Indeed this latter  approach is the one considered  in
  \cite{LinEtal2016Iterative, LeiEtal2021Generalization}. 
  In this case,  the key quantity is the empirical process defined as, 
  \begin{align}
    \sup_{ \|v\| \le R }
    | \widehat{ \mathcal{L} }(v) - \mathcal{L}(v) |.
  \end{align}
  To study the latter, a main complication is that  the iterates norm/path needs
  be bounded almost surely,  and indeed this is a delicate point, as discussed
  in detail in Section~\ref{ssec_GradientConcentration}. 
  In our approach, gradient concentration allows to find a sharp bound on the
  gradient path {\em and at the same time} to directly derive an excess risk bound,
  avoiding the decomposition in~\eqref{eq_ClassicExcessRiskDecomposition} and
  further empirical process bounds. 

{\bf Inexact optimisation and gradient concentration.} 
  We are not the first to employ tools from inexact optimisation 
  to treat learning problems, see \cite{BalakrishnanEtal2017Guarantees} and
  \cite{YangEtal2019EarlyStopping}.
  A similar decomposition as in
  Proposition \ref{prp_DecompositionOfTheExcessRisk} together with a peeling
  argument instead of gradient concentration is used in \cite{YangEtal2019EarlyStopping}.
  The authors, however, derive a bound for a "conditional excess risk".
  More specifically, the risk is the conditional expectation, conditioned on the
  covariates, and is thus still a random quantity.
  The minimizer considered is the minimizer with respect to this random risk and
  therefore is a random quantity too.
  Additionally, their analysis requires strong convexity of the conditional risk
  with respect to the empirical norm.
  Our approach  allows to overcome these two restrictions.\\
  Also gradient concentration has been considered before,   see e.g. \cite{HollandIkeda2018Efficient, PrasadEtal2018Robust}.
In particular, in \cite{HollandIkeda2018Efficient} an analysis is developed under the  assumption that  minimization of  the risk is constrained over a
  closed, convex and bounded set $ \mathcal{W} \subset \mathbb{R}^{d} $, effectively considering an explicit regularization. 
  During their gradient iteration, a projection step is then considered to enforce such a constraint.
  As a consequence the dimension $d$ and the diameter  of $ \mathcal{W} $ appear as key
  quantities that determine the error behavior of their algorithm.
  The same is essentially true for \cite{PrasadEtal2018Robust}.
  In comparison, our analysis is dimension free.  More importantly, however, we do not
 consider any constraint, hence considering implicit, rather than explicit, regularization.
Also,   from a technical point of view this is a key difference. As we discuss in Section \ref{ssec_GradientConcentration}   bounding the gradient path is required, in the absence of explicit constraints.
  The main contribution of our paper, as we see it, is to show that the
  combination of optimisation and concentration of measure techniques presented
  allow to seamlessly control the excess risk and the length of the gradient
  path at the same time, whereas in other analyses, e.g.
  \cite{LeiEtal2021Generalization}, these two tasks have to be separated and are
  much more involved.\\
  Finally, we discuss in detail the results in  \cite{GorbunovEtal2020Stochastic}, of which
  we had not been aware after finishing this work and are closely related. 
  Indeed, also in this paper   inexact optimization and  gradient concentration are combined, albeit in a different way.
  In Theorem G.1., the authors consider stochastic gradient descent for a convex
  and smooth objective function on $ \mathbb{R}^{d} $, notably also on an
  unbounded domain. 
  For their analysis, they introduce clipped versions of the stochastic
  gradients.
  They also borrow a decomposition of the excess risk from inexact optimization,
  although a different one.
  In particular, it is not straightforward that their decomposition would also
  yield results for the last gradient iteration.   
  In a second step, they then use the conditional independence of gradient
  batches and a Bernstein-type inequality for Martingale differences
  to derive concentration for several terms involving gradient noise terms.
  In comparison, instead of concentration based on individual batches, we use the
  full empirical gradients together with a uniform concentration result based on
  Rademacher complexities of Hilbert space valued function classes, see Section
  \ref{ssec_GradientConcentration}. 
  On the one hand, our setting is more general, since we consider a Hilbert
  space instead of $ \mathbb{R}^{d} $. 
  On the other hand, \cite{GorbunovEtal2020Stochastic}  are notably able to
  forgo property \hyperref[ass_RSmooth]{{\color{red} \textbf{(R-Lip)}}}, i.e.
  their gradients can be unbounded.
  This is the main aspect of their analysis.
  As a consequence their result is tailored to this setting
  and does not contain ours as a special case. 
  In particular, with property
  \hyperref[ass_RSmooth]{{\color{red} \textbf{(R-Lip)}}},
  even on $ \mathbb{R}^{d} $, our result is much sharper.
  We avoid an additional $ \log $-factor and, more importantly, we are able to
  freely choose a large, fixed step size $ \gamma > 0 $. 
  In Theorem G.1. of \cite{GorbunovEtal2020Stochastic}, the step size has to depend both on the number of iterations
  and the high probability guarantee of the result.
  Further, our results in Theorem \ref{thm_ExcessRisk} are particularly sharp
  with explicit constants and one clear regularization parameter, $
  \gamma T $, that can, in principle, be chosen via sample splitting and early
  stopping.
  Conversely, in order to control the unbounded gradients
  \cite{GorbunovEtal2020Stochastic} 
  have to introduce two additional hyperparameters:
  the gradient clipping threshold $ \lambda $ and the batch size $ m $.
  In their analysis, both of these have to be chosen in dependence of the true
  minimizer.
  In particular, the clipping threshold $ \lambda $ de facto regularizes the
  problem based on a priori knowledge of the true solution, very much in the way
  as a bounded domain would. Developing these observations is indeed an interesting venue for further research. 
  
{\bf Last iterate vs. averaged iterates convergence.} Finally, we compare our results to other high probability bounds for gradient descent. 
High probability bounds for both last iterate and (tail-)averaged gradient convergence with constant stepsize for least squares regression in Hilbert spaces are well established. 
Indeed, the former follows from \cite{BlanchardMuecke2018Optimal, LinEtal2020Optimal}
as gradient descent belongs to the broader class of spectral
regularisation methods.   Note that this also is well known in the context of inverse problems, see e.g.
  \cite{EnglEtal1996Regularization}. As observed in \cite{MueckeEtal2019Beating}, also average gradient descent can be cast and analyzed in the spectral filtering framework.
Indeed, average and last iterates can be seen to share essentially the same excess risk bound. However, the proof is heavily tailored to least squares. 
Compared to these results, for smooth losses, we establish  a high probability bound of order $\cO(1/\gamma T)$ for uniform averaging and $\cO(\log(T)/\gamma T)$ 
for last iterate GD, for any $n$ sufficiently large, with \emph{constant} stepsize, only worse by a factor of $\log(T)$. 
We note that, it was shown in  \cite{HarveyEtal2019Tight}  that the $\log(T)$ factor is in fact necessary  for Lipschitz functions for 
last iterate  SGD and GD with \emph{decaying} stepsizes. 
Indeed, the authors derive a sharp high probability bound of order $\cO(\log(T)/\sqrt{T})$ for last iterate (S)GD, while uniform averaging 
achieves a faster rate of $\cO(1/\sqrt{T})$.  
Notably,  this work even shows the stronger statement: Any convex combination of the last $k$ iterates must incur a $\log(T/k)$ factor.  
Finally, we note that \cite{LinEtal2016Iterative} derive finite sample bounds for subgradient
descent for convex losses, considering the last iterate. 
In this work, early stopping gives a suboptimal rate, with decaying stepsize and also an additional logarithmic factor. 
This vanishes under additional differentiability and smoothness for constant stepsize.





\section{From inexact optimization to learning}
\label{sec_FromInexactOptimisationToLearning}

In this section, we further discuss the important elements of the proof. 
The alternative  error decomposition we presented in
Proposition~\ref{prp_DecompositionOfTheExcessRisk}  follows from taking the
point of view of optimization with inexact gradients
\cite{BertsekasTsitsiklis2000Gradient}. 
The idea is to consider an ideal GD-iteration subject to noise, i.e. 
\begin{align}
  \label{eq_InexactGradientDescent}
  v_{t + 1} =   v_{t} 
              - \gamma ( \nabla \mathcal{L}( v_{t} ) + e_{t} ) 
  \qquad t = 0, 1, 2, \dots,
\end{align}
where, the $ ( e_{t} )_{t \ge 0} $ are gradient noise terms.
In Equation \eqref{eq_InexactGradientDescent}, very general choices for $ e_{t}
$ may be considered.
Clearly, in  our setting, we have
\begin{align}
  e_{t} =   \nabla \widehat{ \mathcal{L} }( v_{t} ) 
          - \nabla \mathcal{L}( v_{t} ) 
  \qquad t = 0, 1, 2, \dots.
\end{align}
From this perspective,  the empirical GD-iteration can be seen 
as  performing gradient descent directly on the
\emph{population} risk, however, the gradient is corrupted with noise and convergence has to be balanced out with a control of the stability of the iterates.
Next, we see how these ideas  can be applied to the learning problem.


\subsection{Inexact gradient descent}
\label{ssec_InexactGradientDescent}

From the point of view discussed above, it becomes essential to relate both the
risk and the norm of a fixed GD-iteration to the gradient noise.
In the following, we provide two technical Lemmas, which do exactly that.
Both results could also be formulated for general gradient noise terms $ (
e_{t})_{t \ge 0} $.
For the sake of simplicity, however, we opt for the more explicit formulation in
terms of the gradients.
The proofs are based on entirely deterministic arguments and can be found in
Appendix \ref{app_ProofsForFromInexactOptimisationToLearning}. 

\begin{lemma}[Inexact gradient descent:\,Risk]
  \label{lem_InexactGradientDescentRisk}
  Suppose assumptions
  \hyperref[ass_Bounded]{{\color{red} \textbf{(Bound)}}},
  \hyperref[ass_Convex]{{\color{red} \textbf{(Conv)}}},
  and
  \hyperref[ass_Smooth]{{\color{red} \textbf{(Smooth)}}}
  are satisfied.
  Consider the GD-iteration from Definition \ref{def_GradientDescentAlgorithm}
  with constant step size
  $
    \gamma \le 1 / ( \kappa^{2} M )
  $
  and let $ w \in \mathcal{H} $. 
  Then, for any $ t \ge 1 $, the risk of the iterate $ v_{t} $ satisfies 
  \begin{align*}
        \mathcal{L}(v_{t}) - \mathcal{L}(w) 
    \le 
        \frac{1}{ 2 \gamma } 
        ( \| v_{t - 1} - w \|^{2} - \| v_{t} - w \|^{2} )
      + \langle
          \nabla \mathcal{L}( v_{t - 1} ) 
        - 
          \nabla \widehat{ \mathcal{L} }( v_{t - 1} ), 
          v_{t} - w
        \rangle.
  \end{align*}
\end{lemma}

Lemma \ref{lem_InexactGradientDescentRisk} is the key component to
obtain the decomposition of the excess risk in Proposition
\ref{prp_DecompositionOfTheExcessRisk} for the averaged GD-iteration. 
This \emph{online to batch conversion} easily follows by exploiting the
convexity of the population risk \hyperref[ass_RConvex]{{\color{red}
\textbf{(R-Conv)}}}.

The next Lemma is crucial in providing a high probability guarantee for the
boundedness of the gradient path in Proposition \ref{prp_BoundedGradientPath},
which is necessary to apply gradient concentration to the decomposition of the
excess risk in Proposition \ref{prp_DecompositionOfTheExcessRisk}. 

\begin{lemma}[Inexact gradient descent:\,Gradient path]
  \label{lem_InexactGradientDescentGradientPath}
  Suppose assumptions
  \hyperref[ass_Bounded]{{\color{red} \textbf{(Bound)}}},
  \hyperref[ass_Convex]{{\color{red} \textbf{(Conv)}}},
  \hyperref[ass_Lipschitz]{{\color{red} \textbf{(Lip)}}}, 
  \hyperref[ass_Smooth]{{\color{red} \textbf{(Smooth)}}}
  and
  \hyperref[ass_Min]{{\color{red} \textbf{(Min)}}}
  are satisfied and choose a constant step size
  $ 
    \gamma \le \min \{ 1 / ( \kappa^{2} M ) , 1 \}
  $
  in Definition \ref{def_GradientDescentAlgorithm}.
  Then, for any $ t \ge 0 $ the norm of the GD-iterate $ v_{t + 1} $ is
  recursively bounded by
  \begin{align*}
        \| v_{t + 1} - w_{*} \|^{2} 
    & 
    \le 
        \| v_{0} - w_{*} \|^{2} 
    \\ \notag 
    & \ \ \ \
      + 2 \gamma
        \sum_{s = 0}^{t} 
        \Big( 
          \langle 
            \nabla \mathcal{L}( v_{s} ) 
          - \nabla \widehat{ \mathcal{L} }( v_{s} ), 
            v_{s} - w_{*} 
          \rangle 
        + \kappa L 
          \| 
            \nabla \mathcal{L}( v_{s} ) 
          - \nabla \widehat{ \mathcal{L} }( v_{s} )
          \|
        \Big).
  \end{align*}
\end{lemma}

\noindent Assuming that for some fixed $ R > 0 $, 
$ 
  \| v_{s} - w_{*} \| \le R
$ 
for all $ s \le t $, Lemma \ref{lem_InexactGradientDescentGradientPath}
guarantees that 
\begin{align}
        \| v_{t + 1} - w_{*} \|^{2} 
  \le
        \| v_{0} - w_{*} \|^{2} 
      + 2 \gamma ( R + \kappa L )
        \sum_{s = 1}^{t}
        \| 
          \nabla \mathcal{L}( v_{s} ) 
        - \nabla \widehat{ \mathcal{L} }( v_{s} )
        \|,
\end{align}
which, in combination with gradient concentration, allows for an inductive bound
on $ \| v_{t + 1} \| $. 
Summarizing, Lemma \ref{lem_InexactGradientDescentRisk} and Lemma
\ref{lem_InexactGradientDescentGradientPath}  can be regarded as tools to study
our learning problem  using  gradient concentration directly.



\subsection{Gradient concentration}
\label{ssec_GradientConcentration}

In this section, we discuss how the gradient concentration inequality in
Equation~\eqref{eq_grdcoionc} is derived using results in \cite{Foster2018Uniform}.
We use a gradient concentration result which is expressed in terms of the
Rademacher complexity of a function class defined by the gradients
$ w \mapsto \nabla \mathcal{L}(w) $ with 
\begin{align}
    \nabla \mathcal{L}(w) 
  = 
    \int
      \ell'(y, \langle x, w \rangle) x 
    \, \mathbb{P}_{( X, Y )}(d ( x, y )), 
  \qquad w \in \mathcal{H}. 
\end{align}
Since the gradients above are elements of the Hilbert space $ \mathcal{H} $, the
notion of Rademacher complexities has to be stated for Hilbert space-valued
function classes, see \cite{Maurer2016Vector}.

\begin{definition}[Rademacher complexities]
  \label{def_RademacherComplexities} 
  Let $ ( \mathcal{H}, \| \cdot \| ) $ be a real, separable Hilbert space.
  Further, let $ \mathcal{G} $ be a class of maps $ g: \mathcal{Z} \to
  \mathcal{H} $ and $ Z = ( Z_{1}, \dots, Z_{n} ) \in \mathcal{Z}^{n} $ be a
  vector of i.i.d. random variables. 
  We define the \emph{empirical} and \emph{population} Rademacher complexities
  of $ \mathcal{G} $ by
  \begin{align}
    \widehat{\mathcal{R}}_{n}(\mathcal{G}): 
    = \mathbb{E}_{\varepsilon} \Big[ 
        \sup_{g \in \mathcal{G}} 
        \Big\| 
          \frac{1}{n} \sum_{j = 1}^{n} 
          \varepsilon_{j} g(Z_{j})
        \Big\|
      \Big] 
    \qquad \text{ and } \qquad 
    \mathcal{R}_{n}(\mathcal{G}): 
    = \mathbb{E}_{Z} \big[ \widehat{\mathcal{R}}(\mathcal{G}) \big]
  \end{align}
  respectively, where $\varepsilon =(\varepsilon_1 , ..., \varepsilon_n) \in
  \{-1,+1\}^n$ is a vector of i.i.d. Rademacher random variables independent of
  $ Z $.
\end{definition}

\noindent In our setting, we consider
$ 
  ( Z_{1}, \dots, Z_{n} ) 
  =
  ( ( X_{1}, Y_{1} ), \dots, ( X_{n}, Y_{n} ) ) 
$. 
Fix some $ R > 0 $ and consider the scalar function class
\begin{align}
  \label{eq_ScalarFunctionClass}
          \mathcal{F}_{R}: 
  = 
          \{ 
            f_{v} = \langle \cdot, v \rangle: \| v \| \le R
          \} 
  \subset 
          L^{2}(\mathbb{P}_{X}) 
\end{align}
and more importantly, the $ \mathcal{H} $-valued, composite function class
\begin{align}
  \label{eq_CompositeFunctionClass}
    \mathcal{G}_{R}: = \nabla \ell \circ \mathcal{F}_{R}: 
  = 
    \{ 
      \mathcal{Y} \times \mathcal{H}
      \ni ( x, y ) \mapsto \ell'(y, f(x)) x:
      f \in \mathcal{F}_{R}
    \}. 
\end{align}
Under 
\hyperref[ass_Bounded]{{\color{red} \textbf{(Bound)}}}
and 
\hyperref[ass_Lipschitz]{{\color{red} \textbf{(Lip)}}}, 
we have
\begin{align}
  \label{eq_BoundOfTheCompositeFunctionClass}
      G_{R}: 
  = 
      \sup_{ g \in \mathcal{G}_{R} }
      \| g \|_{\infty} 
  = 
      \sup_{ f \in \mathcal{F}_{R} }  
      \| \ell'(Y, f(X)) X \|_{ \infty } 
  \le 
      \kappa L, 
\end{align}
where $ \| \cdot \|_{\infty} $ denotes the $ \infty $-norm on the underlying
probability space $ ( \Omega, \mathscr{F}, \mathbb{P} ) $. 

The gradient concentration result can now be formulated in terms of the
empirical Rademacher complexity of $ \mathcal{G}_{R} $. 

\begin{proposition}[Gradient concentration]
  \label{prp_GradientConcentration}
  Suppose assumption 
  \hyperref[ass_Bounded]{{\color{red} \textbf{(Bound)}}}
  \hyperref[ass_Lipschitz]{{\color{red} \textbf{(Lip)}}}
  and 
  \hyperref[ass_Smooth]{{\color{red} \textbf{(Smooth)}}}
  are satisfied and let $ R > 0 $. 
  Then, for any $ \delta > 0 $, 
  \begin{align*}
        \sup_{\| v \| \le R} 
        \| \nabla \mathcal{L}(v) - \nabla \widehat{\mathcal{L}}(v) \| 
    \le
        4 \widehat{ \mathcal{R} }_{n}( \mathcal{G}_{R} ) 
      + G_{R} \sqrt{ \frac{ 2 \log(4 / \delta) }{n}}
      + G_{R} \frac{ 4 \log(4 / \delta) }{n}
  \end{align*}
  with probability at least $ 1- \delta $, where $ G_R $ is defined in
  Equation \eqref{eq_BoundOfTheCompositeFunctionClass}.
\end{proposition}

\noindent The proof of Proposition \ref{prp_GradientConcentration} is stated in
Appendix \ref{prf_GradientConcentration}. 
To apply Proposition \ref{prp_GradientConcentration}, we need
to bound $
  \widehat{ \mathcal{R} }_{n}( \mathcal{G}_{R} ) 
$.
This can be done relating the empirical Rademacher complexity of the
composite function class $ \mathcal{G}_{R} $ back to the complexity of the
scalar function class $ \mathcal{F}_{R} $. 

\begin{lemma}[Bounds on the empirical Rademacher complexities]
  \label{lem_BoundsOnTheEmpiricalRademacherComplexities}
  Fix $ R > 0 $. Then,
  \begin{enumerate}
    \item [(i)] Under
      \hyperref[ass_Bounded]{{\color{red} \textbf{(Bound)}}},
      we have
      $
          \widehat{ \mathcal{R} }( \mathcal{F}_{R} ) 
        \le
          \frac{\kappa R}{\sqrt{n}} 
      $.

    \item [(ii)] Under
      \hyperref[ass_Bounded]{{\color{red} \textbf{(Bound)}}},
      \hyperref[ass_Lipschitz]{{\color{red} \textbf{(Lip)}}}
      and 
      \hyperref[ass_Smooth]{{\color{red} \textbf{(Smooth)}}}, 
      we have
      \begin{align*}
            \widehat{ \mathcal{R} }( \mathcal{G}_{R} ) 
        \le 
            2 \sqrt{2} 
            \Big( 
              \frac{ \kappa L }{ \sqrt{n} } 
            + \kappa M
              \widehat{ \mathcal{R} }( \mathcal{F}_{R} )
            \Big) 
        \le 
            \frac{ 2 \sqrt{2} ( \kappa L + \kappa^{2} M R ) }
                 { \sqrt{n} }.
      \end{align*}
  \end{enumerate}
\end{lemma}

\noindent Note that since the bounds in Lemma
\ref{lem_BoundsOnTheEmpiricalRademacherComplexities} do not depend on the sample 
$ 
  ( X_{1}, Y_{1} ), \dots, ( X_{n}, Y_{n} )
$, 
they also hold for the population Rademacher complexities.
Lemma \ref{lem_BoundsOnTheEmpiricalRademacherComplexities} (i) is a classic
result, which we restate for completeness.
Lemma \ref{lem_BoundsOnTheEmpiricalRademacherComplexities} (ii) is more involved
and requires combining a vector-contraction inequality from
\cite{Maurer2016Vector} with additional more classical contraction arguments to
disentangle the concatenation in the function class $ \mathcal{G}_{R} $. 
The proof of Lemma \ref{lem_BoundsOnTheEmpiricalRademacherComplexities} is
stated in Appendix \ref{prf_BoundsOnTheEmpiricalRademacherComplexities}.
Note that the arguments for both Proposition \ref{prp_GradientConcentration} and
Lemma \ref{lem_BoundsOnTheEmpiricalRademacherComplexities} are essentially
contained in \cite{Foster2018Uniform}. 
Here, we provide a self-contained derivation for our setting.

Together with Lemma \ref{lem_InexactGradientDescentGradientPath}, the gradient
concentration result provides an immediate high probability guarantee for the
gradient path not to diverge too far from the minimizer $ w_{*} $. 

\begin{proposition}[Bounded gradient path]
  \label{prp_BoundedGradientPath}
  Suppose assumptions
  \hyperref[ass_Bounded]{{\color{red} \textbf{(Bound)}}},
  \hyperref[ass_Convex]{{\color{red} \textbf{(Conv)}}},
  \hyperref[ass_Lipschitz]{{\color{red} \textbf{(Lip)}}},
  \hyperref[ass_Smooth]{{\color{red} \textbf{(Smooth)}}}
  and
  \hyperref[ass_Min]{{\color{red} \textbf{(Min)}}}
  are satisfied, set $ v_{0} = 0 $ and choose a constant step size 
  $ 
    \gamma \le \min \{ 1 / ( \kappa^{2} M ), 1 \} 
  $
  in Definition \ref{def_GradientDescentAlgorithm}. 
  Fix $ \delta \in (0, 1] $ such that 
  \begin{align}
    \label{eq_BoundedGradientPath_nLarge}
        \sqrt{n} 
    \ge 
        \max \{ 1, 90  \gamma T \kappa^{2} ( 1 + \kappa L ) ( M + L ) \} 
        \sqrt{ \log(4 / \delta) } 
  \end{align}
  and 
  $
    R = \max\{ 1, 3 \| w_{*} \| \} 
  $.
  Then, on the gradient concentration event from Proposition
  \ref{prp_GradientConcentration} with
  probability at least $ 1 - \delta $ for the above choice of $ R $, we have 
  \begin{align*}
      \| v_{t} \| \le R
    \qquad \text{ and } \qquad 
      \| v_{t} - w_{*} \| \le \frac{2 R}{3}
    \qquad \text{ for all } t = 1, \dots, T.
  \end{align*}
\end{proposition}

\noindent 
The proof of Proposition \ref{prp_BoundedGradientPath} is stated in Appendix 
\ref{prf_BoundedGradientPath}. 
In a learning setting, bounding the gradient path is essential to the
analysis of gradient descent based estimation procedures.
Either one has to guarantee its boundedness a priori, e.g. by projecting back
onto a ball of known radius $ R > 0 $ or making highly restrictive additional
assumptions, see \cite{LeiTang2018Mirror}, or one has to make usually involved
arguments to guarantee its boundedness up to a sufficiently large iteration
number, see e.g. \cite{LeiEtal2021Generalization} and
\cite{YangEtal2019EarlyStopping}. 
Our numerical illustration in Figure \ref{fig_NumericalIllustrations} shows that
from a practical perspective, such a boundedness result is indeed necessary to
control the variance. 
Additionally, if the boundedness of the gradient path was already controlled by
the optimization procedure for arbitrarily large iterations $ T $, then the
decomposition in Proposition \ref{prp_DecompositionOfTheExcessRisk} together
with our gradient concentration result in Proposition
\ref{prp_GradientConcentration} would guarantee that for $ T \to \infty $, the
deterministic bias part 
$ 
  \| w_{*} \|^{2} / ( 2 \gamma T )
$ 
vanishes completely, while the stochastic variance part
\begin{align}
  \frac{1}{T} \sum_{t = 1}^{T} 
  \langle 
    \nabla \mathcal{L}( v_{t - 1} ) 
  - \nabla \widehat{ \mathcal{L} }( v_{t - 1} ),
    v_{t} - w_{*} 
  \rangle
\end{align}
would remain of order
$ 
  \sqrt{\log(4 / \delta) / n}
$
independently of $ T $.
This would suggest that for large $ T $ there is no tradeoff between reducing
the bias of the estimation method and its variance anymore, which, in that form,
should be surprising for  learning, see discussion in
\cite{DerumignySchmidtHieber2020LowerBounds}. 
From this perspective, in order to analyze gradient descent for learning, it
seems needed to establish a result like Proposition
\ref{prp_BoundedGradientPath}.  

We compare our result in Proposition \ref{prp_BoundedGradientPath} with the
corresponding result in \cite{LeiEtal2021Generalization}, which is the most
recent in the literature.
Under the \emph{self-boundedness} assumption 
\begin{align}
  | \ell'(y, a) |^{2} \lesssim \ell(y, a) + 1 
  \qquad \text{ for all } y, a \in \mathbb{R}, 
\end{align}
they relate the stochastic gradient descent iteration $ v_{t} $ to the Tikhonov regularizer 
$ w_{\lambda} $, whose norm can be controlled and obtain a uniform bound over 
$ t = 1, \dots, T $ of the form
\begin{align}
  \label{eq_BoundedGradientPathInLeiEtal}
            \| v_{t + 1} \|^{2} 
  \lesssim 
            \sum_{s = 1}^{t} 
            \gamma_{s} 
            \max 
            \{ 
              0,
              \widehat{ \mathcal{L} }( w_{\lambda} )
            - \widehat{ \mathcal{L} }( v_{s} )
            \} 
          + \log( 2 T / \delta ) ( \| w_{\lambda} \|^{2} + 1 )
\end{align}
with high probability in $ \delta $. 
Later, the risk quantities in Equation \eqref{eq_BoundedGradientPathInLeiEtal}
are related to the approximation error of a kernel space, which guarantees that
the stochastic gradient path is sufficiently bounded.
For the bound in Equation \eqref{eq_BoundedGradientPathInLeiEtal},
\cite{LeiEtal2021Generalization} have to choose a decaying sequence of step sizes
$ \gamma_{t} $ with 
$
  \sum_{t = 1}^{T} \gamma_{t}^{2} < \infty
$.
In comparison, the result in Proposition \ref{prp_BoundedGradientPath} allows
for a fixed step size $ \gamma > 0 $.
Since sharp rates essentially require that $ \sum_{t = 0}^{T} \gamma_{t}
$ is of order $ \sqrt{n} $, we may therefore stop the algorithm earlier.
In this regard, our result is a little sharper.
At the same time, the result in \cite{LeiEtal2021Generalization} is more
general.
In fact, under a capacity condition, the authors adapt the bound in Equation
\eqref{eq_BoundedGradientPathInLeiEtal} to allow for fast rates.
However, both the proof of Equation \eqref{eq_BoundedGradientPathInLeiEtal} and
its adaptation to the capacity dependent setting are very involved and quite
technical.
In comparison, Proposition \ref{prp_BoundedGradientPath} is an immediate
corollary of Proposition \ref{prp_GradientConcentration}. 
In particular, if under additional assumptions, a sharper concentration
result for the gradients is possible, our proof technique would immediately
translate this to the bound on the gradient path that is needed to guarantee
this sharper rate for the excess risk. Indeed, we think these ideas can be fruitfully developed
to get new improved results.



\section{Numerics}
\label{sec_Numerics}

In this section, we provide some empirical illustration to the effects described
in Section \ref{sec_MainResultsAndDiscussion} and Section
\ref{sec_FromInexactOptimisationToLearning}.
In particular, we choose the logistic loss for regression from Example
\ref{expl_LossFunctionsSatisfyingTheAssumptions} (b) and concentrate on two
aspects: The (un)bounded gradient path for the averaged iterates and the
interplay between step size and stopping time.
Our experiments are conducted on synthetic data with $ d = 100 $ dimensions,
generated as follows:
We set the covariance matrix $ \Sigma \in \mathbb{R}^{ d \times d }$ as a
diagonal matrix with entries $ \Sigma_{ j j } = j^{-2} $, $ j = 1, \dots, d $
and choose $ w_{*} = \Sigma e $, with 
$ e = ( 1, \dots, 1 )^{ \top } \in \mathbb{R}^{d} $.
We generate $ n_{ \text{train} } =10,000 $ training data, where the covariates $
X_{j} $ are drawn from a Gaussian distribution with zero mean and covariance $
\Sigma $.
For $ j = 1, \dots ,n_{ \text{train} }$, the labels follow the model $ 
  Y_{j} = \langle X_{j}, w_{ * } \rangle + \varepsilon_{ j }
$ with 
$ \varepsilon_{ j } $ being standard Gaussian noise.
Each experiment is repeated $ 100 $ times and we report the average.  

Our first experiment illustrates the behavior of the path 
$ t \mapsto || v_t - w_* || $ for $ t \in \{ 1, ...., 1000 \}$ and 
$ \gamma = 1 $. 
As Proposition \ref{prp_BoundedGradientPath} suggest, this path becomes
unbounded if the number of iterations grows large. 

In a second experiment, 
we choose a grid of step sizes $ \gamma \in \{2, ...., 10\} $ and stopping times 
$T \in \{1, ..., 1000\}$ and report the average excess test 
risk with $n_{ \text{test} } = n_{ \text{train} } /3 $ test data.
The result is presented in the right plot in Figure \ref{fig:plots}.  As Theorem
\ref{thm_ExcessRisk} predicts, for fixed $n_{ \text{train} }$, the performance
of averaged GD remains roughly constant as $\gamma \cdot T$ remains constant.  

\begin{figure}[h]
  \label{fig_NumericalIllustrations} 
  \centering
  \includegraphics[width=0.44\textwidth]{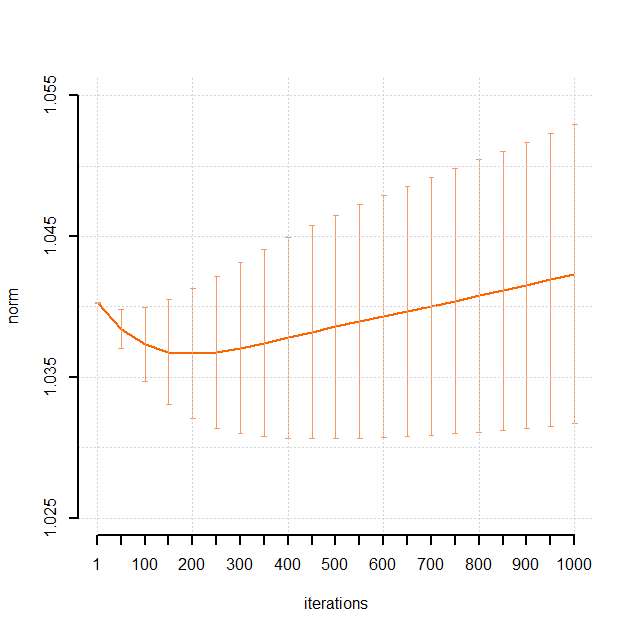} 
  \includegraphics[width=0.44\textwidth]{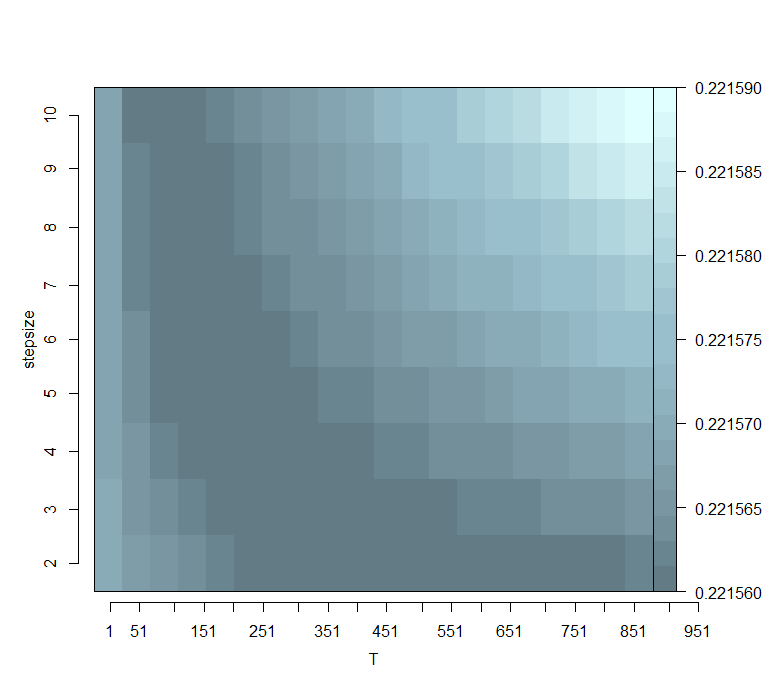}
  \caption{
    Left:  The gradient path  becomes unbounded for a growing number of iterations.   
    Right: Excess risk for logistic loss as a function of  $T$ and
           $\gamma$ for averaged GD.
  }
  \label{fig:plots}
\end{figure}


\section{Conclusion}
\label{sec_Conclusion}

In this paper, we studied implicit/iterative regularization for linear, possibly infinite dimensional, 
linear models, where the error cost is a convex, differentiable loss function. Our main contribution is a 
sharp high probability bound on the excess risk of the averaged and last iterate of batch gradient descent. 
We derive these results combining ideas and results from optimization and statistics. Indeed, we show how it is 
possible to leverage results from inexact optimization together with concentration inequalities for vector valued functions. 
The theoretical results are illustrated to see how the step size and the iteration number control the bias and the stability of the solution. \\
A number of research directions can further be developed.
In our study we favored a simple analysis to illustrate the main ideas, and as a consequence our results are limited to a basic setting. Indeed, it would be interesting to develop the analysis we presented to get faster learning rates under further assumptions, 
for example considering capacity conditions or even finite dimensional models. Another possible research direction is to  consider less 
regular loss functions, in particular dropping the differentiability assumption. Along similar lines it would be interesting to consider other form of implicit bias or non linear models.
Finally, it would be interesting to consider other forms of optimization including stochastic and accelerated methods.



\subsubsection*{Acknowledgements} 
The authors would like to thank Silvia Villa and Francesco Orabona for useful discussions. 

The research of B.S. has been partially funded by the Deutsche
Forschungsgemeinschaft (DFG)- Project-ID 318763901 - SFB1294.

N.M. acknowledges funding by the Deutsche Forschungsgemeinschaft (DFG)
under Excellence Strategy \emph{The Berlin Mathematics Research Center MATH+} 
(EXC-2046/1, project ID:390685689).

L.R. acknowledges support from  the Center for Brains, Minds and Machines (CBMM), funded by NSF STC award CCF-1231216.  L.R. also acknowledges the financial support of the European Research Council (grant SLING 819789), the AFOSR projects FA9550-18-1-7009, FA9550-17-1-0390 and BAA-AFRL-AFOSR-2016-0007 (European Office of Aerospace Research and Development), and the EU H2020-MSCA-RISE project NoMADS - DLV-777826.


\bibliographystyle{alpha}
\bibliography{bib_SGD}


\newpage

\appendix

\section{Appendix: Proofs for Section \ref{sec_MainResultsAndDiscussion} }
\label{app_ProofsForMainResults}

\begin{lemma}[Properties of the risks]
  \label{lem_PropertiesOfTheRisks}
  \
  \begin{enumerate}

    \item [(i)] Under \hyperref[ass_Convex]{\textbf{{\color{red} (Conv)}}}, 
      the population risk convex, i.e., for all $ v, w \in \mathcal{H} $ we have 
      \begin{align}
        \label{eq:convexity3}
        \mathcal{L}(v) 
        \le 
        \mathcal{L}(w)
        - 
        \langle \nabla \mathcal{L}(v), w - v \rangle. 
      \end{align}

    \item [(ii)]
      Under 
      \hyperref[ass_Bounded]{\textbf{{\color{red} (Bound)}}}
      and
      \hyperref[ass_Lipschitz]{\textbf{{\color{red} (Lip)}}},
      the population  risk is Lipschitz-continuous with
      constant $ \kappa L $, i.e.,  for all $v, w \in \cH$ we have
      \begin{align}
        \label{eq:lip3}
        | \mathcal{L}( v ) - \mathcal{L}( w ) | 
        \le 
        \kappa L \| v - w \| 
      \end{align}

    \item [(iii)]
      Under 
      \hyperref[ass_Bounded]{\textbf{{\color{red} (Bound})}}
      and
      \hyperref[ass_Smooth]{\textbf{{\color{red} (Smooth)}}},
      the gradient of the population  risk is Lipschitz-continuous with constant
      $ \kappa^{2}  M $, i.e., for all $ v, w \in \mathcal{H} $ we have 

      \begin{align}
        \label{eq:smooth3-1}
        \| 
          \nabla \mathcal{L}( v ) - \nabla \mathcal{L}( w )
        \| 
        & \le 
        \kappa^{2} M \| v - w \| 
      \end{align}
      Note that this implies that  
      \begin{align}
        \label{eq:smooth3-2}
        \mathcal{L}( w ) 
        & \le 
        \mathcal{L}( v )
        + 
        \langle \nabla \mathcal{L}( v ), w - v \rangle 
        + 
        \frac{ \kappa^{2} M }{ 2 } 
        \| w - v \|^{2}. 
      \end{align}

  \end{enumerate}
  Moreover, $(i), (ii)$ and $(iii)$ also hold for the empirical risk 
  $ \widehat{ \mathcal{L} } $ with the same constants. 
\end{lemma}

\begin{proof}[Proof]
  \
  \begin{enumerate}
    \item [(i)] This follows directly from
      \hyperref[ass_Convex]{\textbf{{\color{red} (Conv)}}}
      and the linearity of the expectation.

    \item [(ii)]
      For $ v, w \in \mathcal{H} $, we have
      \begin{align}
            | \mathcal{L}(w) - \mathcal{L}(v) | 
        & 
        = 
            | 
              \mathbb{E}_{( X, Y )} [ \ell(Y, \langle X, w \rangle) ] 
            - \mathbb{E}_{( X, Y )} [ \ell(Y, \langle X, v \rangle) ] 
            | 
        \\ \notag
        & 
        \le 
            L \mathbb{E}_{( X, Y )} [ \| X \| \| w - v \| ] 
        \le 
            \kappa L \| w - v \|,
      \end{align}
      where the first inequality follows from
      \hyperref[ass_Lipschitz]{\textbf{{\color{red} (Lip)}}}
      and Cauchy-Schwarz inequality and the second inequality follows from 
      \hyperref[ass_Bounded]{\textbf{{\color{red} (Bound})}}.

    \item [(iii)]
      For $ v, w \in \mathcal{H} $, we have
      \begin{align}
            \| \nabla L(w) - \nabla \mathcal{L}(v) \| 
        & 
        = 
            \| 
              \mathbb{E}_{( X, Y )} [ \ell'(Y, \langle X, w \rangle) X ] 
            - \mathbb{E}_{( X, Y )} [ \ell'(Y, \langle X, v \rangle) X ] 
            \| 
        \\ \notag 
        &
        \le
            \kappa 
            | 
              \mathbb{E}_{( X, Y )} [ \ell'(Y, \langle X, w \rangle) ] 
            - \mathbb{E}_{( X, Y )} [ \ell'(Y, \langle X, v \rangle) ] 
            |
        \\ \notag 
        &
        \le 
            \kappa M
            | \mathbb{E}_{( X, Y )} [ \langle X, w - v \rangle ] |
        \le 
            \kappa^{2} M \| w - v \|,
      \end{align}
      where the first and the third inequality follow from 
      \hyperref[ass_Bounded]{\textbf{{\color{red} (Bound})}} and Cauchy-Schwarz
      inequality and the second one follows from 
      \hyperref[ass_Smooth]{\textbf{{\color{red} (Smooth)}}}.
  \end{enumerate}
\end{proof}

For the proof of the second part of Proposition
\ref{prp_DecompositionOfTheExcessRisk} we need the following simple Lemma. 
A different version of this was put forward in a blog post by Francesco Orabona
with a reference to the convergence proof of the last iterate of SGD in
\cite{LinEtal2016Iterative}. 
Since our version is different and for the sake of completeness, we give a full
proof.

\begin{lemma} 
  \label{lem:omg-so-simple}
  Let $ (q_t)_{ t = 1,...,T } $ be a sequence of real numbers.
  Then, 
  \begin{align*}
          q_{T} 
    & = 
            \frac{1}{T} \sum_{ t = 1 }^{T} q_{t} 
          + 
            \sum_{ t = 1 }^{ T - 1 } 
            \frac{1}{ t ( t + 1 ) } 
            \sum_{ s = T - t + 1 }^{T} 
            ( q_{s} - q_{ T - t } ). 
  \end{align*}
\end{lemma}

\begin{proof}[Proof]
  Define
  \begin{align}
    S_{t}: 
    = 
    \frac{1}{t} \sum_{ s = T - t + 1 }^{T} q_{s}, 
    \qquad t = 1, \dots, T.
  \end{align}
  Then, any $ t \le T - 1 $ satisfies
  \begin{align}
            t S_{t} 
    & = 
            ( t + 1 ) S_{ t + 1 } - q_{ T - t } 
    = 
            t S_{ t + 1 } + S_{ t + 1 } - q_{ T - t } 
          \\
    & = 
            t S_{ t + 1 } 
          + 
            \frac{1}{ t + 1 } 
            \sum_{ s = T - t }^{ T } 
            ( q_{s} - q_{ T - t } ), 
          \notag
  \end{align}
  which implies
  \begin{align}
    \label{eq_InductionStep}
            S_{t} 
    & = 
            S_{ t + 1 } 
          + 
            \frac{1}{ t ( t + 1 ) } 
            \sum_{ s = T - t }^{ T } 
            ( q_{s} - q_{ T - t } ). 
  \end{align}
  Inductively applying \eqref{eq_InductionStep}, we obtain
  \begin{align}
            q_{T} 
    = 
            S_{1} 
    & = 
            S_{T} 
          + 
            \sum_{ t = 1 }^{ T - 1 } 
            \frac{1}{ t ( t + 1 ) } 
            \sum_{ s = T - t }^{ T } 
            ( q_{s} - q_{ T - t } ). 
          \\
    & = 
            \frac{1}{T} \sum_{ t = 1 }^{T} q_{t} 
          + 
            \sum_{ t = 1 }^{ T - 1 } 
            \frac{1}{ t ( t + 1 ) } 
            \sum_{ s = T - t + 1 }^{ T } 
            ( q_{s} - q_{ T - t } ). 
            \notag
  \end{align}
\end{proof}

\begin{proof}[\textbf{Proof of Proposition} \ref{prp_DecompositionOfTheExcessRisk} (Decomposition of the excess risk)]
  \label{prf_DecompositionOfTheExcessRisk}
  \
  \begin{enumerate}
    \item [(i)] From \hyperref[ass_RConvex]{{\color{red} \textbf{(R-Conv)}}} and
      Lemma \ref{lem_InexactGradientDescentRisk}, we obtain
      \begin{align}
        & \ \ \ \
        \frac{1}{T} \sum_{t = 1}^{T}
        \mathcal{L}(v_{t}) - \mathcal{L}(w)
        \\ \notag
        &
        \le 
            \frac{1}{T} \sum_{t = 1}^{T}
              \frac{\| v_{t - 1} - w \|^{2} - \| v_{t} - w \|^{2}}{2 \gamma}
          + 
            \frac{1}{T} \sum_{t = 1}^{T}
              \langle 
                \nabla \mathcal{L}( v_{t - 1} ) 
              - 
                \nabla \widehat{\mathcal{L}}( v_{t - 1} ), 
                v_{t} - w
              \rangle
        \\ \notag
        &
        \le 
            \frac{\| v_{0} - w \|^{2}}{2 \gamma T}
          + 
            \frac{1}{T} \sum_{t = 1}^{T}
              \langle 
                \nabla \mathcal{L}(v_{t - 1}) 
              - 
                \nabla \widehat{\mathcal{L}}(v_{t - 1}), 
                v_{t} - w
              \rangle, 
      \end{align}
      where have resolved the telescopic sum, to obtain the last inequality. 

    \item[(ii)] Applying Lemma \ref{lem:omg-so-simple} with $ 
        q_{t} 
         = 
        \mathcal{L}( v_{t} ) - \mathcal{L}(w)
      $, we find  
      \begin{align}
        \label{eq:ex2}
            \mathcal{L}( v_{T} ) - \mathcal{L}(w) 
        & = 
              \frac{1}{T} 
              \sum_{ t = 1 }^{T} 
              ( \mathcal{L}( v_{t} ) - \mathcal{L}(w) ) 
            + 
              \sum_{ t = 1 }^{ T - 1 } 
              \frac{1}{ t ( t + 1 ) } 
              \sum_{ s = T - t + 1 }^{T} 
              ( \mathcal{L}( v_{s} ) - \mathcal{L}( v_{ T - t } ) ). 
      \end{align}
      We aim at bounding the last sum in the above equality. Summing the bound
      in Lemma \ref{lem_InexactGradientDescentRisk} from $T-t+1$ to T yields for
      all $ v \in \mathcal{H} $. 
      \begin{align}
        \sum_{ s = T - t + 1 }^{T} 
        \mathcal{L}( v_{s} ) - \mathcal{L}(v) 
        & \le 
        \frac{1}{ 2 \gamma } 
        \| v_{ T - t } - v \|^{2} 
        + 
        \sum_{ s = T - t + 1 }^{T} 
        \langle 
          \nabla \mathcal{L}( v_{ s - 1 } ) 
          - 
          \nabla \widehat{ \mathcal{L} }_{ s - 1 }, 
          v_{s} - v
        \rangle.
      \end{align}
      Hence, setting $ v = v_{T-t} $ yields 
      \begin{align}
        \sum_{ s = T - t + 1 }^{ T } 
        ( \mathcal{L}( v_{ s } ) - \mathcal{L}( v_{ T - t } ) ) 
        & \le 
        \sum_{ s = T - t + 1 }^{ T } 
        \langle 
          \nabla \mathcal{L}( v_{ s - 1 } ) 
        - \nabla \widehat{ \mathcal{L} }( v_{ s - 1 } ), 
          v_{ s } - v
        \rangle. 
      \end{align}
      The result follows by plugging the last inequality into \eqref{eq:ex2}. 
  \end{enumerate}
\end{proof}

\begin{proof}[\textbf{Proof of Theorem} \ref{thm_ExcessRisk} (Excess risk)]
  \label{prf_ExcessRisk} 
  We initially consider the case of the averaged GD-iterate.
  By convexity, Proposition \ref{prp_DecompositionOfTheExcessRisk} and an application
  of Cauchy-Schwarz inequality, we have
  \begin{align}
    \label{eq_ExcessRisk_Decomposition}
        \mathcal{L}( \overline{v}_{T} ) - \mathcal{L}( w_{*} ) 
    & 
    \le \frac{1}{T} \sum_{t=1}^T \mathcal{L}( v_t ) - \mathcal{L}(w) \\ \notag 
    &
    \le
        \frac{ \| w_{*} \|^{2} }
             { 2 \gamma T } 
      + \frac{1}{T} \sum_{t = 1}^{T}  
        \langle 
          \nabla \mathcal{L}( v_{t} ) 
        - \nabla \widehat{ \mathcal{L} }( v_{t} ),
          v_{t} - w_{*} 
        \rangle 
    \\ \notag 
    & 
    \le
        \frac{ \| w_{*} \|^{2} }{ 2 \gamma T } 
      + \frac{1}{T} \sum_{t = 1}^{T} 
        \| 
          \nabla \mathcal{L}( v_{t} ) 
        - \nabla \widehat{ \mathcal{L} }( v_{t} )
        \| 
        \| v_{t} - w_{*} \|. 
  \end{align}
  The assumptions of Theorem \ref{thm_ExcessRisk} are chosen exactly as in
  Proposition \ref{prp_BoundedGradientPath}. 
  Therefore, on the gradient concentration event with probability at least $ 1 -
  \delta $ from Proposition \ref{prp_GradientConcentration} and the choice $ R $
  as above, we have
  \begin{align}
        \| v_{t} - w_{*} \| \le \frac{2 T}{3}, 
    \ \ 
        \| v_{t} \| \le R 
    \ \ \text{ and } \ \
        \| 
            \nabla \mathcal{L}( v_{t} ) 
          - \nabla \widehat{ \mathcal{L} }( v_{t} ) 
        \| 
    \le 
        20 \kappa^{2} R ( L + M ) 
        \sqrt{ \frac{ \log(4 / \delta) }{n} }
    \\ \notag 
    \qquad \text{ for all } t = 0, 1, \dots, T,
  \end{align}
  where the last inequality is derived in exactly the same way as in
  the proof of Proposition \ref{prp_BoundedGradientPath}. 
  Plugging this into the inequality in Equation
  \eqref{eq_ExcessRisk_Decomposition}, we obtain
  \begin{align}
        \mathcal{L}( \overline{v}_{T} ) - \mathcal{L}( w_{*} ) 
    & 
    \le 
        \frac{\| w_{*} \|^{2}}{2 \gamma T} 
      + 20 \kappa^{2} R^{2} ( M + L ) 
        \sqrt{ \frac{ \log(4 / \delta) }{n} }. 
    \\ \notag 
    & 
    \le 
        \frac{ \| w_{*} \|^{2} } { 2 \gamma T } 
      + 180 \max \{ 1, \| w_{*} \|^{2} \} 
        \kappa^{2} ( M + L )
        \sqrt{ \frac{\log(4 / \delta) }{n} }. 
  \end{align}

  For the last iterate, we set $ 
        e_{t}: 
    =
        \nabla \widehat{ \mathcal{L} }( v_{t} ) 
      - \widehat{ \mathcal{L} }( v_{t} )
  $, $ t = 1, \dots, T $ to reduce the notation.
  Proposition \ref{prp_DecompositionOfTheExcessRisk} with an application of
  Cauchy-Schwarz yields
  \begin{align}
            \mathcal{L}( v_{T} ) - \mathcal{L}( w_{*} ) 
    & \le 
            \frac{1}{T} 
            \sum_{ t = 1 }^{T} 
            ( \mathcal{L}( v_{t} ) - \mathcal{L}( w_{*} ) ) 
          + 
            \sum_{ t = 1 }^{ T - 1 } 
            \frac{1}{ t ( t + 1 ) } 
            \sum_{ s = T - t + 1 }^{T} 
            \langle 
              - e_{ s - 1 }, v_{s} - v_{ T - t }
            \rangle 
            \\
    & \le 
            \frac{ \| w_{*} \|^{2} }{ 2 \gamma T } 
          + 
            \frac{1}{T} 
            \sum_{ t = 1 }^{T} 
            \langle - e_{ s - 1 }, v_{s} - v_{ T - t } \rangle 
            \notag 
            \\
          & \qquad \qquad \qquad \ + 
            \sum_{ t = 1 }^{ T - 1 } 
            \frac{1}{ t ( t + 1 ) } 
            \sum_{ s = T - t + 1 }^{T} 
            \langle 
              - e_{ s - 1 }, v_{s} - v_{ T - t }
            \rangle 
            \notag
            \\
    & \le 
            \frac{ \| w_{*} \|^{2} }{ 2 \gamma T } 
            + 
            \frac{1}{T} 
            \sum_{ t = 1 }^{T} 
            \| e_{s} \| \| v_{s} - v_{*} \| 
            \notag 
            \\
    & \qquad \qquad \qquad \ + 
            \sum_{ t = 1 }^{ T - 1 } 
            \frac{1}{ t ( t + 1 ) } 
            \sum_{ s = T - t + 1 }^{T} 
            \| e_{s} \| \| v_{s} - v_{ T - t } \|. 
            \notag
  \end{align}

  Now, by Proposition \ref{prp_GradientConcentration} and Proposition
  \ref{prp_BoundedGradientPath}, if 
  \begin{align}
    \sqrt{n} 
    & \ge 
    90 \gamma T \kappa^{2}
    ( 1 + \kappa L ) ( M + L ) 
    \sqrt{ \log( 4 / \delta ) } 
  \end{align}
  we find with probability at least $ 1 - \delta $ for all 
  $ t = 0, \dots, T $, that
  \begin{align}
    \| w_{*} \| \le \frac{ 2 R }{3} \le R, 
    \qquad 
    \| v_{t} \| \le R, 
  \end{align}
  \begin{align}
            \| e_{t} \| 
    & \le 
            \sup_{ v \in \mathcal{F}_{R} } 
            \| 
              \nabla \mathcal{L}(v) 
              -
              \nabla \widehat{ \mathcal{L} }(v)
            \| 
    \le 
            20 \kappa^{2} R ( L + M ) 
            \sqrt{ \frac{ \log( 4 / \delta ) }{n} }. 
  \end{align}
  In particular, 
  \begin{align}
    \| v_{s} - v_{ T - t } \| \le \frac{4 R}{3} 
  \end{align}
  for any $ 
    s = T - t + 1, \dots T, 
    t = 1, \dots, T
  $. 
  Hence, with probability at least $ 1 - \delta $, 
  \begin{align}
            \mathcal{L}( v_{T} ) - \mathcal{L}( w_{*} ) 
    & \le 
            \frac{ \| w_{*} \|^{2} }{ 2 \gamma T } 
          + 
            20 \kappa^{2} R^{2} ( L + M ) 
            \sqrt{ \frac{ \log( 4 / \delta ) }{n} } 
            \notag 
            \\
          & \qquad \qquad \qquad + 
            \frac{ 80 }{3} 
            \kappa^{2} R^{2} ( L + M ) 
            \sqrt{ \frac{ \log( 4 / \delta ) }{n} } 
            \Big( \sum_{ t = 1 }^{T} \frac{1}{t} \Big).
            \notag 
  \end{align}
  Finally, since 
  \begin{align}
          \sum_{ t = 1 }^{ T - 1 } \frac{1}{t} 
    & \le 
          \log( T - 1 ) 
    \le 
          \log T,
  \end{align}
  we arrive at 
  \begin{align}
            \mathcal{L}( v_{ T } ) - \mathcal{L}( w_{ * } )
    & \le 
            \frac{ \| w_{ * } \|^{2} }{ 2 \gamma T } 
          + 
            \frac{140}{3} 
            \kappa^{2} R^{2} ( L + M ) 
            \sqrt{ \frac{ \log( 4 / \delta ) }{n} } 
            \log T.
  \end{align}
  Plugging in for $ R^{2} $ and simplifying the constant completes the proof.
\end{proof}


\section{Appendix: Proofs for Section \ref{sec_FromInexactOptimisationToLearning}}
\label{app_ProofsForFromInexactOptimisationToLearning}

\begin{proof}[\textbf{Proof of Lemma} \ref{lem_InexactGradientDescentRisk} (Inexact gradient descent:\,Risk)]
  \label{prf_InexactGradientDescentRisk}
  By 
  \hyperref[ass_RConvex]{{\color{red} \textbf{(R-Conv)}}} (equation \eqref{eq:convexity3})
  and 
  \hyperref[ass_RSmooth]{{\color{red} \textbf{(R-Smooth)}}} (equation \eqref{eq:smooth3-2}), 
  the population risk is convex and $ \kappa^{2} M $-smooth.
  We have
  \begin{align}
        \mathcal{L}(v_{t})
    &
    \le 
        \mathcal{L}(v_{t - 1})
      + 
        \langle \nabla \mathcal{L}(v_{t - 1}), v_{t} - v_{t - 1} \rangle
      + 
        \frac{\kappa^{2} M}{2} \| v_{t} - v_{t - 1} \|^{2}
    \\ \notag
    &
    \le 
        \mathcal{L}(w)
      + 
        \langle \nabla \mathcal{L}(v_{t - 1}), v_{t - 1} - w \rangle
      + 
        \langle \nabla \mathcal{L}(v_{t - 1}), v_{t} - v_{t - 1} \rangle
      + 
        \frac{\kappa^{2} M}{2} \| v_{t} - v_{t - 1} \|^{2}
    \\ \notag
    &
    \le 
        \mathcal{L}(w)
      +   
        \langle \nabla \mathcal{L}(v_{t - 1}), v_{t} - w \rangle
      +   
        \frac{1}{2\gamma } \| v_{t} - v_{t - 1} \|^{2},
  \end{align}
  where the last inequality uses from the fact that 
  $
    \gamma \le 1 / ( \kappa^{2} M ) 
  $.
  The statement now follows from
  \begin{align}
        \| v_{t} - v_{t - 1} \|^{2}
    = 
        \| v_{t - 1} - w \|^{2} - \| v_{t} - w \|^{2}
      - 2 \gamma
        \langle
          \nabla \widehat{\mathcal{L}}(v_{t - 1}), v_{t} - w
        \rangle.
  \end{align}
\end{proof}

\begin{proof}[\textbf{Proof of Lemma} \ref{lem_InexactGradientDescentGradientPath} (Inexact gradient descent:\,Gradient path)]
  \label{prf_InexactGradientDescentGradientPath} 
  For $ v, w \in \mathcal{H} $, Equation (3.6) in \cite{Bubeck2015Optimization}
  together with \hyperref[ass_RSmooth]{{\color{red} \textbf{(R-Smooth)}}} (equation \eqref{eq:smooth3-1})
  yields
  \begin{align}
    \label{eq_InexactGradientDescentGradientPath_BubeckInequality}
    \| \nabla \mathcal{L}(v) - \nabla \mathcal{L}(w) \|^{2} 
    \le 
    \kappa^{2} M 
    \langle 
      v - w, 
      \nabla \mathcal{L}(v) - \nabla \mathcal{L}(w)
    \rangle 
  \end{align}
  In particular, since $ \nabla \mathcal{L}( w_{*} ) = 0 $, we have
  \begin{align}
    \| \nabla \mathcal{L}(v) \|^{2} 
    \le 
    \kappa^{2} M 
    \langle 
      v - w_{*}, \nabla \mathcal{L}(v)
    \rangle 
    \qquad \text{ for all } v \in \mathcal{H}. 
  \end{align}
  Setting 
  $ 
    e_{s}: =  \nabla \widehat{\mathcal{L}}( v_{s} ) 
            - \nabla \mathcal{L}( v_{s} )
  $,
  we obtain that for any $ s \ge 0 $,
  \begin{align}
        \| v_{s + 1} - w_{*} \|^{2} 
    = 
    &
        \| v_{s} - w_{*} \|^{2} 
      - 2 \gamma \langle 
                   \nabla \widehat{ \mathcal{L} }( v_{s} ),
                   v_{s} - w_{*} 
                 \rangle 
      + \gamma^{2} \| \nabla \widehat{ \mathcal{L} }( v_{s} ) \|^{2} 
    \\ \notag 
    = 
    &
        \| v_{s} - w_{*} \|^{2} 
      - 2 \gamma \langle e_{s}, v_{s} - w_{*} \rangle 
    \\ \notag 
    & 
    \underbrace{
      - 2 \gamma \langle 
      \nabla \mathcal{L}( v_{s} ),
      v_{s} - w_{*} 
      \rangle 
      + \gamma^{2} \| \nabla \mathcal{L}( v_{s} ) \|^{2} 
    }_{= (\text{I})}
    \\ \notag 
    & 
    \underbrace{
      + \gamma^{2} \| \nabla \widehat{ \mathcal{L} }( v_{s} ) \|^{2} 
      - \gamma^{2} \| \nabla \mathcal{L}( v_{s} ) \|^{2} 
    }_{= (\text{II})}. 
  \end{align}
  We treat the terms $ (\text{I}) $ and $ (\text{II}) $ separately:
  By Equation \eqref{eq_InexactGradientDescentGradientPath_BubeckInequality} and
  our choice of $ \gamma \le 1 / ( \kappa^{2} M ) $, we have
  \begin{align}
    \label{eq_InexactGradientDescentGradientPath_I}
        (\text{I})
    & 
    = 
        - 2 \gamma \langle 
        \nabla \mathcal{L}( v_{s} ),
        v_{s} - w_{*} 
        \rangle 
        + \gamma^{2} \| \nabla \mathcal{L}( v_{s} ) \|^{2} 
    \\ \notag 
    &
    \le 
        \Big( 
          \frac{- 2 \gamma}{\kappa^{2} M} + \gamma^{2}
        \Big) 
        \| \nabla \mathcal{L}( v_{s} ) \|^{2} 
    \le 0. 
  \end{align}

  Further, by \hyperref[ass_RLipschitz]{\textbf{{\color{red} (R-Lip)}}} (equation \eqref{eq:lip3}),
  Cauchy-Schwarz inequality and the fact that $ \gamma \le 1 $, 
  \begin{align}
    \label{eq_InexactGradientDescentGradientPath_II}
        (\text{II})
    & 
    = 
          \gamma^{2} \| \nabla \widehat{ \mathcal{L} }( v_{s} ) \|^{2} 
        - \gamma^{2} \| \nabla \mathcal{L}( v_{s} ) \|^{2} 
    = 
        \gamma^{2} 
        \langle 
          \nabla \widehat{ \mathcal{L} }( v_{s} ) 
        + \nabla \mathcal{L}( v_{s} ), 
          e_{t} 
        \rangle 
    \\ \notag 
    &
    \le 
        \gamma^{2} \| 
                     \nabla \widehat{ \mathcal{L} }( v_{s} ) 
                   + \nabla \mathcal{L}( v_{s} ) 
                   \| 
                   \| e_{s} \| 
    \le 2 \gamma \kappa L \| e_{s} \|. 
  \end{align}
  Together, Equation \eqref{eq_InexactGradientDescentGradientPath_I} and
  \eqref{eq_InexactGradientDescentGradientPath_II} yield
  \begin{align}
        \| v_{s + 1} - w_{*} \|^{2} - \| v_{s} - w_{*} \|^{2} 
    \le 
        - 2 \gamma \langle 
                     v_{s} - w_{*}, e_{s}
                   \rangle 
        + 2 \gamma \kappa L \| e_{s} \|. 
  \end{align}
  Summing over $ s $ then yields the result.
\end{proof}

\begin{proof}[\textbf{Proof of Lemma} \ref{lem_BoundsOnTheEmpiricalRademacherComplexities} (Bounds on the empirical Rademacher complexities)]
  \label{prf_BoundsOnTheEmpiricalRademacherComplexities} 
  The first statement of Lemma
  \ref{lem_BoundsOnTheEmpiricalRademacherComplexities} is a classical result,
  see e.g. \cite{BartlettMendelson2002Complexities}. 

  For the second statement, recall that for any $ w \in \mathcal{H} $, we have 
  $ 
    \| w \| = \sup_{\| v \| = 1} \langle v, w \rangle 
  $, 
  since $ \mathcal{H} $ is assumed to be real.
  Thus, we may write 
  \begin{align}
    \label{eq_RademacherComplexityOfTheCompositeFunctionClass_DotProductFormulationOfTheNorm}
    \widehat{ \mathcal{R} }_{n}( \mathcal{G}_{R} ) 
    & 
    = 
    \mathbb{E}_{\varepsilon} 
    \Big[ 
      \sup_{ \nabla \ell \circ f \in \mathcal{G}_{R} } 
      \Big\| 
        \frac{1}{n} 
        \sum_{j = 1}^{n} 
          \varepsilon_{j} 
          \ell'( Y_{j}, f(X_{j}) ) X_{j}
      \Big\|
    \Big] 
    \\ \notag 
    & 
    = 
    \mathbb{E}_{\varepsilon} 
    \Big[ 
      \sup_{ f \in \mathcal{F}_{R} } 
      \sup_{\| v \| = 1} 
      \frac{1}{n} 
      \sum_{j = 1}^{n} 
        \varepsilon_{j} 
        \ell'( Y_{j}, f(X_{j}) ) 
        \langle X_{j}, v \rangle 
    \Big]. 
  \end{align}
  In order to bound the right-hand side in Equation 
  \eqref{eq_RademacherComplexityOfTheCompositeFunctionClass_DotProductFormulationOfTheNorm},
  we apply Theorem 2 from \cite{Maurer2016Vector}.
  Adopting the notation from this result, we may restrict the supremum above to
  a countable dense subset $ \mathcal{S} $ of 
  $ 
    \mathcal{F}_{R} \times \{ v \in \mathcal{H}: \| v \| \le 1 \}
  $. 
  Note that this is possible, since by \hyperref[ass_Smooth]{{\color{red}
  \textbf{(Smooth)}}}, we have that $ \ell' $ is continuous in the second
  argument.
  Further, set
  \begin{align}
    \psi_{j} & : \mathcal{S} \to \mathbb{R},
    \qquad 
      \psi_{j}(f, v):
    =
      \ell'(Y_{j}, f(X_{j})) 
      \langle X_{j}, v \rangle;
    \\ \notag
    \phi_j^{(1)} & : \mathcal{S} \to \mathbb{R},
    \qquad  
      \phi_j^{(1)}(f, v):
    =
      L  \langle X_{j}, v \rangle;
    \\ \notag
    \phi_j^{(2)} & : \mathcal{S} \to \mathbb{R},
    \qquad
      \phi_j^{(2)}(f, v):
    =
      \kappa \ell'( Y_{j}, f( X_{j} ) ).
    \\ \notag 
    \phi_{j} & : \mathcal{S} \to \mathbb{R}^{2}, 
    \qquad 
      \phi_{j}(f, v) 
    = 
      \big( \phi_{j}^{(1)}(f, v), \phi_{j}^{(2)}(f, v) \big)
  \end{align}
  Then, for any $ j = 1, \dots, n $, and 
  $
    ( f, v ), ( g, w ) \in \mathcal{S} 
  $,
  we use that 
  $
    || \ell' ||_{\infty} \le L 
  $
  by \hyperref[ass_Lipschitz]{{\color{red} \textbf{(Lip)}}},
  $ \| X_{j} \| \le \kappa $ 
  by \hyperref[ass_Bounded]{{\color{red} \textbf{(Bound)}}}
  and $ \| w \| \le 1 $ to obtain
  \begin{align}
    \label{eq_BoundsOnTheEmpiricalRademacherComplexities_MaurerAssumption}
        | \psi_{j}(f, v) - \psi_{j}(g, w) | 
    &
    =
        | 
          \ell'( y_{j}, f( X_{j} ) ) \langle X_{j}, v \rangle
        - \ell'( y_{j}, g( X_{j} ) ) \langle X_{j}, w \rangle
        |  
        \\ \notag
    & 
    \le
        | \ell'(Y_{j}, f(X_{j})) \langle X_{j}, v - w \rangle | 
      + |
          ( \ell'(Y_{j}, f(X_{j})) - \ell'(Y_{j}, g(X_{j})) )
          \langle X_{j}, w \rangle
        | 
        \\ \notag
    &
    \le
        | L \langle X_{j}, v \rangle  - L \langle X_{j}, w \rangle | 
      + | \kappa \ell'(Y_{j}, f(X_{j})) - \kappa \ell'(Y_{j}, g(X_{j})) |
        \\ \notag
    &
    = 
        \| \phi_{j}(f, v) - \phi_{j}(g, w) \|_{1, \mathbb{R}^{2}}
        \\ \notag
    & 
    \le
        2 \| \phi_{j}(f, v) - \phi_{j}(g, w) \|_{2, \mathbb{R}^{2}}, 
  \end{align}
  where $ \| \cdot \|_{p, \mathbb{R}^{2}} $ denotes the $ p $-norm on 
  $ \mathbb{R}^{2} $.
  Equation 
  \eqref{eq_BoundsOnTheEmpiricalRademacherComplexities_MaurerAssumption} 
  shows that Theorem 2 from \cite{Maurer2016Vector} is in fact applicable,
  which yields
  \begin{align}
    \label{eq_RademacherComplexityOfTheCompositeFunctionClass_SplitBound}
        \widehat{\mathcal{R}}( \mathcal{G}_{R} ) 
    & 
    \le 
        \underbrace{ 
          2 \sqrt{2} 
          \mathbb{E}_{\varepsilon} 
          \sup_{f \in \mathcal{F}_{R}}
          \sup_{\| v \| = 1} 
          \frac{1}{n}
          \sum_{j = 1}^{n} 
            \varepsilon_{j}  \phi_j^{(1)}(f , v)
        }_{=: (I)}
    \\ \notag
    & 
      + \underbrace{
          2 \sqrt{2}  
          \mathbb{E}_{\varepsilon} 
           \sup_{f \in \mathcal{F}_{R}}
           \sup_{\| v \| = 1} 
           \frac{1}{n}
           \sum_{j = 1}^{n} \varepsilon_j 
             \phi_j^{(2)}(f , v)
        }_{=: (II)}.
  \end{align}

  We proceed by bounding each term individually. 
  Applying Theorem 7 from \cite{Foster2018Uniform} with $ \beta = 1 $ and 
  $ \Psi(v) = \| v \|^{2} / 2 $ leads to
  \begin{align}
    \label{eq_RademacherComplexityOfTheCompositeFunctionClass_I}
      (I)
    & 
    = 
      2 \sqrt{2} L
      \mathbb{E}_{\varepsilon} 
      \Big[ 
        \sup_{ f \in \mathcal{F}_{R} } 
        \sup_{ \| v \| = 1 } 
        \frac{1}{n} \sum_{j = 1}^{n} 
        \varepsilon_{j} \langle X_{j}, v \rangle 
      \Big] 
    \le 
      2 \sqrt{2} L 
      \mathbb{E}_{\varepsilon} 
      \Big[ 
        \frac{1}{n} \sum_{j = 1}^{n} 
        \varepsilon_{j} \| X_{j} \|
      \Big] 
    \\ \notag 
    &
    \le 
      \frac{ 2 \sqrt{2} L }{n} 
      \sqrt{ \sum_{j = 1}^{n} \| X_{j} \|^{2} } 
    \le 
      \frac{ 2 \sqrt{2} \kappa L }{ \sqrt{n} },
  \end{align}
  where we have used again that $ \| X_{j} \| \le \kappa $.

  For the second summand, Talagrand's contraction principle,
  see e.g. Exercise 6.7.7 in \cite{Vershynin2018HDProb}, 
  together with the fact that by 
  by \hyperref[ass_Smooth]{{\color{red} \textbf{(Smooth)}}},
  $ \ell' $ is $ M $-Lipschitz yields the bound
  \begin{align}
    \label{eq_RademacherComplexityOfTheCompositeFunctionClass_II}
        (II) 
    \le 
        2 \sqrt{2} \kappa M
        \widehat{ \mathcal{R} }( \mathcal{F} ) 
    \le 
        \frac{ 2 \sqrt{2} \kappa^{2} M R }{ \sqrt{n} }, 
  \end{align}
  due to the first part of this Lemma.
  Together, Equation
  \eqref{eq_RademacherComplexityOfTheCompositeFunctionClass_I} and
  \eqref{eq_RademacherComplexityOfTheCompositeFunctionClass_II} yield the
  result.
\end{proof}

\begin{proof}[\textbf{Proof of Proposition} \ref{prp_GradientConcentration} (Gradient concentration)]
  \label{prf_GradientConcentration}
  For $ (x, y) \in \mathcal{H} \times \mathcal{Y} $,
  $ f \in \mathcal{F} _{R}  $ denote  
  \begin{align}
    g_{f}(x, y) 
    = ( \nabla \ell \circ f )(x, y) 
    = \ell'(y, f(x)) x. 
  \end{align}
  Then $g_f \in \cG_R$.
  Applying McDiarmid's bounded difference inequality, see e.g. Corollary 2.21 in
  \cite{Wainwright2019HDStats}, we obtain that on an event with probability at
  least $ 1 - \delta $,
  \begin{align}
    \label{eq:grad-diff}
    \sup_{f \in \mathcal{F}_{R}} 
    \| \nabla \mathcal{L}(f) - \nabla \widehat{\mathcal{L}}(f) \| 
    & 
    = \sup_{g \in \mathcal{G}_{R}} \Big\| 
        \mathbb{E} [ g_{f}(X, Y) ] 
      - \frac{1}{n} \sum_{j = 1}^{n} g_{f}(X_{j}, Y_{j})
      \Big\| 
    \\ \notag 
    & 
    \le \mathbb{E} \Big[ 
          \sup_{g \in \mathcal{G}_{R}} \Big\| 
            \mathbb{E} [ g_{f}(X, Y) ] 
          - \frac{1}{n} \sum_{j = 1}^{n} 
            g_{f}(X_{j}, Y_{j})
          \Big\| 
        \Big]
      + G_{R} \sqrt{ \frac{ 2 \log(2 / \delta) }{n} }.
  \end{align} 
  Applying Lemma 4 from \cite{Foster2018Uniform}, we obtain
  \begin{align}
        \mathbb{E}
        \Big[ 
          \sup_{g \in \mathcal{G}_{R}}
          \Big\| 
            \mathbb{E} [ g_{f}(X, Y) ] 
          - \frac{1}{n} \sum_{j = 1}^{n} g_{f}(X_{j}, Y_{j})
          \Big\| 
        \Big]
    \le 
        4 \widehat{ \mathcal{R} }_{n}( \mathcal{G}_{R} ) 
      + 4 G_{R} \frac{\log(2 / \delta)}{n} 
  \end{align}
  on an event with probability at least $ 1 - \delta $. 
  A union bound finally yields
  \begin{align}
        \sup_{f \in \mathcal{F}_{R}} 
        \| \nabla \mathcal{L}(f) - \nabla \widehat{\mathcal{L}}(f) \| 
    \le
        4 \widehat{ \mathcal{R} }_{n}( \mathcal{G}_{R} ) 
      + G_{R} \sqrt{ \frac{ 2 \log(4 / \delta) }{n} }
      + G_{R} \frac{ 4 \log(4 / \delta) }{n}. 
  \end{align}
  on an event of at least probability $ 1 - \delta $. 
\end{proof}

\begin{proof}[\textbf{Proof of Proposition} \ref{prp_BoundedGradientPath} (Bounded gradient path)]
  \label{prf_BoundedGradientPath}
  Firstly, note that
  \begin{align}
        \| v_{t} \| 
    \le 
        \| v_{t} - w_{*} \| + \| w_{*} \| 
    \le 
        \| v_{t} - w_{*} \| + \frac{R}{3}. 
  \end{align}
  Therefore, it suffices to prove that 
  $
    \| v_{t} - w_{*} \| \le 2 R / 3 
  $ 
  on the gradient concentration event.
  We proceed via induction over $ t \le T $.
  For $ t = 0 $, this is trivially satisfied, since 
  $
    \| v_{1} - w_{*} \| = \| w_{*} \| \le R / 3 
  $.
  Now, assume that the result is true for 
  $
    s = 0, \dots, t < T
  $.
  From Lemma \ref{lem_InexactGradientDescentGradientPath}, we have
  \begin{align}
    \label{eq_BoundedGradientPath_FundamentalEstimate}
    \notag
        \| v_{t + 1} - w_{*} \|^{2} 
    & \le 
        \| w_{*} \|^{2} 
      + 
        2 \gamma \sum_{s = 0}^{t} 
        \Big( 
          \langle
            \nabla \mathcal{L}( v_{s} ) 
          - \nabla \widehat{ \mathcal{L} }( v_{s} ), 
            v_{s} - w_{*}
          \rangle
        + \kappa L
          \| 
            \nabla \mathcal{L}( v_{s} ) 
          - \nabla \widehat{ \mathcal{L} }( v_{s} ), 
          \| 
        \Big) 
    \\ \notag 
    & \le 
        \frac{ R^{2} }{9} 
      + 
        2 \gamma
        \sum_{s = 0}^{t} 
        ( \| v_{s} - w_{*} \| + \kappa L )
          \|
            \nabla \mathcal{L}( v_{s} ) 
          - \widehat{\mathcal{L}}( v_{s} )
          \| 
    \\
    & \le 
        \frac{ R^{2} }{9} 
      + 
        2 \gamma T 
        \Big( 
          \frac{2 R}{3} + \kappa L
        \Big) 
        \sup_{ v \in \mathcal{F}_{R} } 
        \| 
          \nabla \mathcal{L}(v) 
        - 
          \nabla \widehat{ \mathcal{L} }(v)
        \|,
  \end{align}
  where $ \mathcal{F}_{R} $ is defined in Equation
  \eqref{eq_ScalarFunctionClass}. 

  On the gradient concentration event from Proposition
  \ref{prp_GradientConcentration}, by Lemma
  \ref{lem_BoundsOnTheEmpiricalRademacherComplexities} (ii), we have the bound
  \begin{align}
        \sup_{ v \in \mathcal{F}_{R} } 
        \| 
          \nabla \mathcal{L}(v) 
        - 
          \nabla \widehat{ \mathcal{L} }(v)
        \|
    &
    \le 
        \frac{ 8 \sqrt{2} ( \kappa L + \kappa^{2} M R ) }
             { \sqrt{n} } 
      + 
        G_{R} 
        \sqrt{ \frac{ 2 \log(4 / \delta) }{n} } 
      + 
        G_{R} 
        \frac{ 4 \log(4 / \delta) }{n}, 
  \end{align}
  where 
  $ 
    G_{R} \le \kappa L
  $
  by Equation \eqref{eq_BoundOfTheCompositeFunctionClass}. 
  Equation \eqref{eq_BoundedGradientPath_nLarge} guarantees that
  \begin{align}
      \frac{ \log(4 / \delta) }{n} 
    \le 
      \sqrt{ \frac{\log(4 / \delta) }{n} }
  \end{align}
  and hence with 
  $ 
    1 \le \sqrt{ \log(4 / \delta) }
  $, on the gradient concentration event, we obtain that 
  \begin{align}\label{eq:grad-event}
      \sup_{v \in \mathcal{F}_{R}} 
      \| 
        \nabla \mathcal{L}(v) 
      - 
        \nabla \widehat{ \mathcal{L} }(v)
      \|
    \le 
      20 \kappa^{2} R ( L + M ) 
      \sqrt{ \frac{ \log(4 / \delta) }{n} }.
  \end{align}
  Plugging the last bound into Equation
  \eqref{eq_BoundedGradientPath_FundamentalEstimate} yields
  \begin{align}
        \| v_{t} - w_{*} \|^{2} 
    \le 
        \frac{ R^{2} }{9} 
      + 40 \gamma T 
      \kappa^{2} \Big( \frac{2 R^{2}}{3} + \kappa R L \Big) 
      ( M + L ) 
        \sqrt{ \frac{ \log(4 / \delta) }{n} }.
  \end{align}
  Hence, we obtain our result when the second term above is smaller than 
  $ 4 R ^{2} / 9 $, which is satisfied when
  \begin{align}
      \sqrt{n} 
    \ge 
      90 \gamma T \kappa^{2} ( 1 + \kappa L ) ( M + L ) 
      \sqrt{ \log(4 / \delta) },
  \end{align}
  where we have used the fact that $ R \ge 1 $. 
  This completes the proof.
\end{proof}


\checknbnotes

\end{document}